\documentclass[letterpaper]{article} 
\usepackage{aaai25}  
\usepackage{times}  
\usepackage{helvet}  
\usepackage{courier}  
\usepackage[hyphens]{url}  
\usepackage{graphicx} 
\usepackage{natbib}  
\usepackage{caption} 
\usepackage{marvosym}
\usepackage{xcolor}
\usepackage{url}

\usepackage{mathtools}
\usepackage{mathtools}
\usepackage[toc,page]{appendix}

\DeclarePairedDelimiter\abs{\lvert}{\rvert}%

\usepackage{amscd}
\usepackage{lipsum}
\usepackage{tabularx}
\usepackage{amsfonts}
\usepackage{amsmath}
\usepackage{formal-grammar}
\usepackage{amssymb}
\usepackage{amsthm}
\usepackage{cases}		 
\usepackage{cutwin}
\usepackage{enumerate}
\usepackage{epstopdf}
\usepackage{graphicx}
\usepackage{mathrsfs}	
\usepackage{multimedia}
\usepackage{caption}
\usepackage{subcaption}
\usepackage{float}
\usepackage{examplep}
\usepackage{mathtools}

\usepackage{stmaryrd}
\usepackage{verbatim}

\usepackage{thm-restate}
\usepackage{tkz-euclide}
\usepackage{algorithm2e}
\SetKwComment{Comment}{/* }{ */}
\RestyleAlgo{ruled}

\makeatletter
\newcounter{program}

\makeatother

\usepackage{listings}
\definecolor{codegreen}{rgb}{0,0.6,0}
\definecolor{codegray}{rgb}{0.5,0.5,0.5}
\definecolor{codepurple}{rgb}{0.58,0,0.82}
\definecolor{backcolour}{rgb}{0.95,0.95,0.92}

\lstdefinestyle{mystyle}{
    backgroundcolor=\color{backcolour},   
    commentstyle=\color{codegreen},
    keywordstyle=\color{magenta},
    numberstyle=\tiny\color{codegray},
    stringstyle=\color{codepurple},
    basicstyle=\ttfamily\footnotesize,
    breakatwhitespace=false,         
    breaklines=true,                 
    captionpos=b,                    
    keepspaces=true,                 
    numbers=left,                    
    numbersep=5pt,                  
    showspaces=false,                
    showstringspaces=false,
    showtabs=false,                  
    tabsize=2
}

\usepackage{tikz}
\usetikzlibrary{automata, positioning, arrows}

\tikzset{
->, 
>=stealth, 
node distance=1.5cm, 
every state/.style={thick}, 
initial text=$ $, 
}

\newcommand{\N}{\mathbb{N}}

\newcommand{\leftact}{\textnormal{\texttt{LEFT}}}
\newcommand{\rightact}{\textnormal{\texttt{RIGHT}}}
\newcommand{\upact}{\textnormal{\texttt{UP}}}
\newcommand{\downact}{\textnormal{\texttt{DOWN}}}

\newcommand{\edges}{\textnormal{\texttt{Edges}}}
\newcommand{\regions}{\textnormal{\texttt{Regions}}}
\newcommand{\actions}{\textnormal{\texttt{Actions}}}
\newcommand{\adj}{\textnormal{\texttt{Adj}}}

\newcommand{\PEnv}{\mathscr{P}_{\text{Env}}}
\newcommand{\PPol}{\mathscr{P}_{\text{Pol}}}

\title{Programmatic Reinforcement Learning:\\ Navigating Gridworlds}
\author{
    Guruprerana Shabadi\textsuperscript{\rm 1,\rm 3},
    Nathana{\"e}l Fijalkow\textsuperscript{\rm 2,\rm 3},
    Th{\'e}o Matricon\textsuperscript{\rm 2}
}
\affiliations{
  \textsuperscript{\rm 1}University of Pennsylvania, United States,\\
  \textsuperscript{\rm 2}CNRS, LaBRI, Université of Bordeaux, France,\\
  \textsuperscript{\rm 3}University of Warsaw, Poland
}

\usepackage{amssymb,amsthm,amsmath}

\newtheorem{theorem}{Theorem}
\newtheorem{example}{Example}

\newtheorem{lemma}{Lemma}

\begin{document}

\maketitle

\begin{abstract}
The field of reinforcement learning (RL) is concerned with algorithms for learning optimal policies in unknown stochastic environments.
Programmatic RL studies representations of policies as programs, meaning involving higher order constructs such as control loops.
Despite attracting a lot of attention at the intersection of the machine learning and formal methods communities, very little is known on the theoretical front about programmatic RL: what are good classes of programmatic policies? How large are optimal programmatic policies? How can we learn them?
The goal of this paper is to give first answers to these questions, initiating a theoretical study of programmatic RL.
Considering a class of gridworld environments, we define a class of programmatic policies.
Our main contributions are to place upper bounds on the size of optimal programmatic policies, and to construct an algorithm for synthesizing them.
These theoretical findings are complemented by a prototype implementation of the algorithm.
\end{abstract}

%

\section{Introduction}
\label{sec:introduction}
Reinforcement Learning (RL) is a very popular and successful field of machine learning where the agent learns a policy in an unknown environment through numerical rewards, modelled as a Markov decision process (MDP). In the tabular setting, the environment is given explicitly, which implies that typically policies are also represented explicitly, meaning as functions mapping each state to an action (or distribution of actions). Such a representation becomes quickly intractable when the environment is large and makes it hard to compose policies or reason about them.
In the general setting, the typical assumption is that the environment can be simulated as a black-box. Deep reinforcement learning algorithms which learn policies in the form of large neural networks have been scaled to achieve expert-level performance in complex board and video games \cite{DeepMindChessGo, DeepMindStarCraft2}. 
However, they suffer from the same drawbacks as neural networks which means that the learned policies are vulnerable to adversarial attacks \cite{MinimalisticAdversarialAttacks} and do not generalize to novel scenarios \cite{Snderhauf2018LimitsDLRobotics}. Moreover, big neural networks are very hard to interpret and their verification is computationally infeasible.

To alleviate these pitfalls, a growing body of work has emerged which aims to learn policies in the form of programs~\cite{VermaMSKC18, BastaniPS18, Verma0YC19, ZhuXMJ19, Inala2020SynthesizingPP, Landajuela2021DiscoveringSP, Trivedi2021LearningTS, Qiu2022ProgrammaticRL, Andriushchenko2022FSCPOMDP, Liang2022LLMCodePolicies}, under the name ``\textit{programmatic reinforcement learning}''. Programmatic policies can provide concise representations of policies which are easier to read, interpret, and verify. Furthermore, their short size compared to neural networks would mean that they can also generalize well to out-of-training situations while also smoothing out erratic behaviors.
The goal of the line of work we initiate here is to lay the theoretical foundations for programmatic reinforcement learning.

\vskip1em
A programming language defined for a specific set of tasks is commonly called a Domain Specific Language (DSL).
All the works cited above use very simple DSLs, combining finite state machines, decision trees, and Partial Integral Derivate (PID) controllers (originating from control theory). We believe -- and show evidence in this work -- that more expressive DSLs can help describing policies succinctly and naturally. The fundamental question we ask is:

\begin{quotation}
Given a class $\PEnv$ of environments, how to define a DSL $\PPol$ such that
for each environment in $\PEnv$, there exists an optimal\footnote{Different notions of optimality can be considered here.} programmatic policy $\sigma \in \PPol$.
\end{quotation}
Designing a DSL means choosing appropriate programming paradigm, control operators, and primitives.
The design of $\PPol$ is a compromise: the DSL should be rich enough to express optimal policies, but simple enough to meet the objectives stated above: readable, interpretable, verifiable, generalizable, as well as learnable.
We call such statements ``\textit{expressivity results}''. We find many instances of expressivity results originating from different fields:
\begin{itemize}
	\item In the study of games, positionality results for MDPs~\cite{mdpchapter,gamesbook} state the existence of optimal pure memoryless policies.
	\item In program verification, more specifically in the analysis of pushdown systems~\cite{BouajjaniEM97,pushdownchapter}, there are several results proving the existence of optimal pushdown policies.
	\item In automatic control, PID controllers~\cite{PID} form a very classical and versatile class of programmatic controllers widely used in practice for continuous systems.
	\item In machine learning, the fact that neural networks are universal approximators implies that neural networks can be used as (approximate) programmatic policies for general reinforcement learning tasks.
\end{itemize}
Once existence is understood, one can wonder about sizes: how large are optimal programmatic policies? We refer to results placing upper and lower bounds on sizes of optimal programmatic policies as ``\textit{succinctness results}''.

\paragraph*{Contributions.}
In this work we focus on (two dimensional) gridworlds, which is a classical example in RL and inspired from theoretical robotics.
We construct a very simple and elegant class of programmatic policies in the form of sequences of subgoals, inspired by Shannon's early experiments on mechanical mice. 
Our main contribution is a theoretical result: we prove the existence of optimal programmatic policies (an \textit{expressivity} result), and place upper bounds on their sizes (a \textit{succinctness} result). 
To the best of our knowledge, this is the first example of a non-trivial DSL, employing at its heart control loops, which can express optimal policies in a succinct and natural way on a large class of environments.

Together with this article we release a small Python package including modules for generating instances of the environments we study as well as implementation of all the algorithms defined here. See the appendix for more information.

\paragraph*{Outline.}
We define gridworlds in Section~\ref{sec:gridworld}.
We introduce our DSL, called the ``subgoal DSL'', in Section~\ref{sec:shannon}, which defines the class of programmatic policies.
Our algorithm for constructing programmatic policies follows two steps:
first in Section~\ref{sec:tree} we define an algorithm for constructing the tree of shortest paths,
and second in Section~\ref{sec:deriving_policy_programs} we define a second algorithm for extracting from the tree a programmatic policy.
The most technical proofs can be found in the appendix, together with implementation details and experiments.

\section{Gridworlds}
\label{sec:gridworld}
Gridworlds form a very classical example of environments in reinforcement learning and beyond.
The particular class of gridworlds we consider here is inspired by theoretical robotics, and closely resembles for instance~\cite{DegenerKLHPW11}.
The same model (with minor variations) was studied to model hybrid systems~\cite{AsarinMP95}.

We consider gridworlds in two dimensions: the state space is $[0,\ell]^2 \subset \mathbb{R}^2$.
It is divided into closed convex polygonal regions, which we refer to later as the \textit{regions}.
In each region we specify a convex cone of available actions.
A move inside a region consists in picking an available action and moving along it within the region.
We also fix an initial state and a target region: the goal is from the initial state to reach any state in the target region.

\begin{example}[Spiral]\label{ex:spiral}
\begin{figure}
    \centering
    \begin{subfigure}[b]{0.45\linewidth}
	    \centering
	    \includegraphics[width=\linewidth]{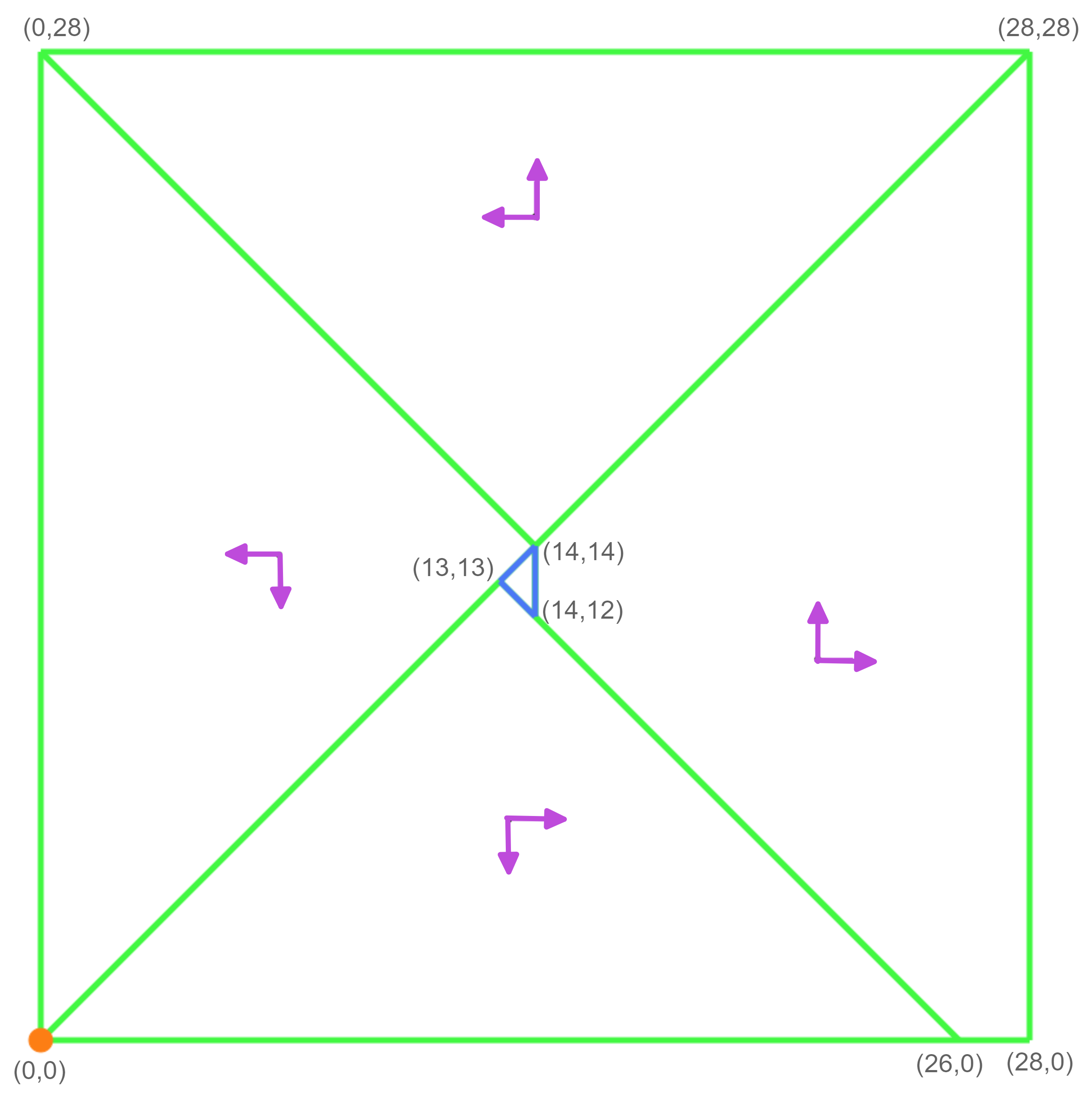}
	    \caption{Gridworld}
	    \label{fig:spiral-gridworld}
    \end{subfigure}
    \begin{subfigure}[b]{0.45\linewidth}
        \centering
        \includegraphics[width=\linewidth]{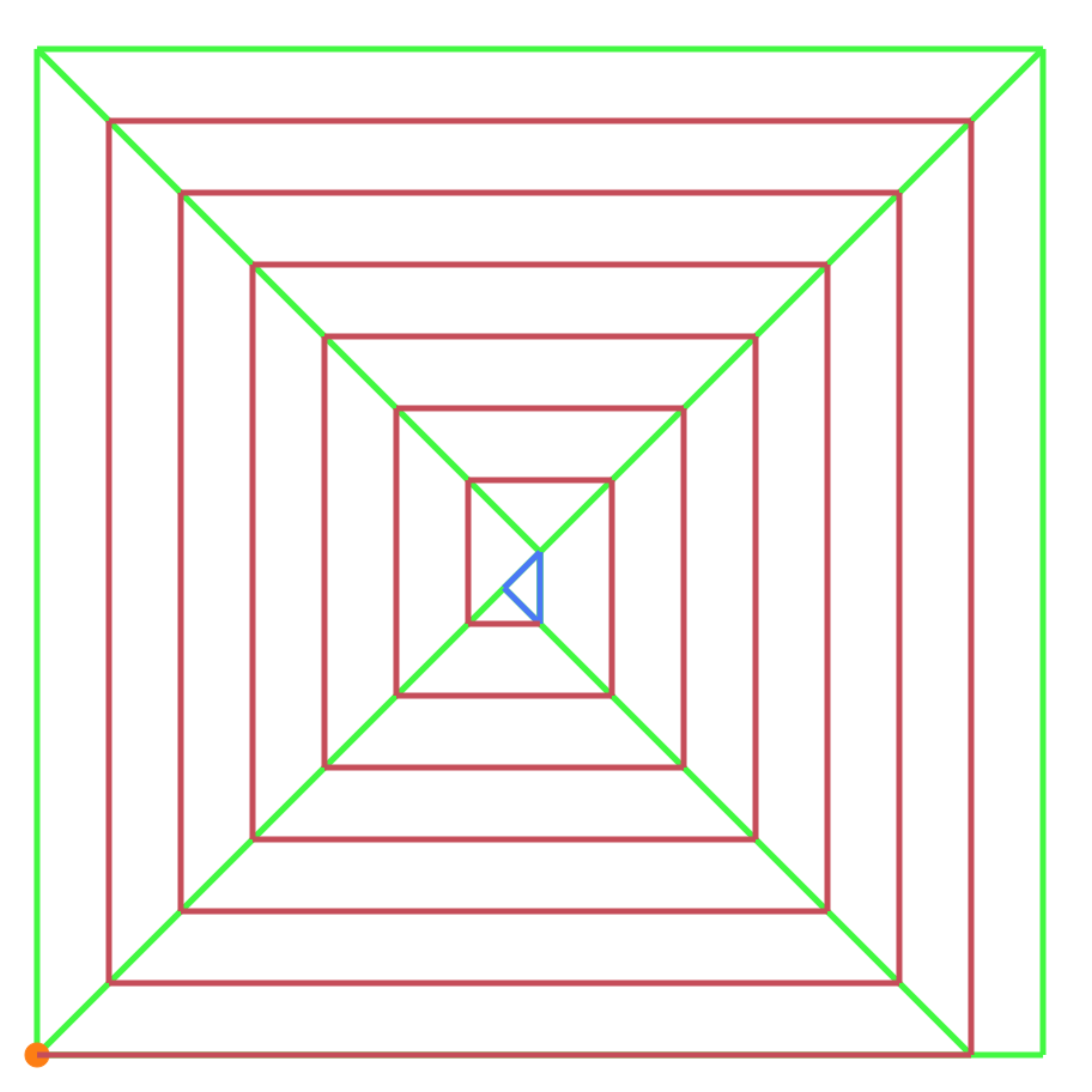}
        \caption{The optimal policy}
        \label{fig:spiral-policy}
    \end{subfigure}
    \caption{Spiral}
    \label{fig:spiral}
\end{figure}
In Figure \ref{fig:spiral-gridworld} we have an example of a gridworld. 
There are $5$ regions with the convex cone of available actions indicated by arrows within the region. The initial state (in orange) is \((0,0)\) and the target region is the triangular region in the middle. Figure \ref{fig:spiral-policy} visualizes the path from the initial state to the target region, spiralling around~it. 
\end{example}

Formally\footnote{We remark that since we have a continuous state space, the value of $\ell$ can always be scaled and it does not play a role in the presented results.}, states are pairs $(x,y) \in [0,\ell]^2$. For any \(v_1, v_2 \in [0,\ell]^2\), we define the segment \([v_1, v_2]\) to be the set of points connecting \(v_1\) and \(v_2\). 
We let $\regions$ denote the set of regions, which are closed convex polygons: we assume that $\bigcup_{R \in \regions} R = [0,\ell]^2$ and any two different regions only intersect on edges.

We let $\actions(R)$ denote the set of actions available in region $R$, it is a convex cone in $\mathbb{R}^2$ generated by a subset of the four cardinal directions \(\texttt{LEFT}, \texttt{RIGHT}, \texttt{UP}\), and \(\texttt{DOWN}\).
We say that there exists a move between $v_1$ and $v_2$ if they belong to the same region and $v_2 - v_1 \in \actions(R)$.
By extension, there exists a path from $v_1$ to $v_k$ if there exists a sequence of consecutive moves starting in $v_1$ and ending in $v_k$.
The length of a path is its number of moves\footnote{Note that this notion of length is different from the total length of the segments describing the path.}.

For a region \(R \in \regions\), we use \(\texttt{Edges}(R)\) to refer to the set of edges of~$R$. 
The convexity assumptions imply that we can restrict our attention to moves between edges:

\begin{lemma}
If $v_1$ and $v_2$ belong to the same region and there exists a path from $v_1$ to $v_2$ inside this region, then there exists a move between $v_1$ and $v_2$.
\end{lemma}

Note that edges are segments, but when introducing an edge we implicitly mean the edge of some region.
Since each edge is shared by at most two regions, we can define \(\adj(R, [v, v']) \in \regions\) to be the region adjacent to \(R\) which both share the edge \([v, v']\) on their boundaries. However, since some edges lie on the boundaries of the grid \([0, \ell]^2\), \(\adj(R, [v, v'])\) might not exist.

Note that gridworlds are deterministic environments, which implies that our goal is to find a minimal \textit{path} from the initial state to the target region, where minimal means of minimal length. Thus, we will be using path and policy interchangeably moving forward. 

The first direction we explore while searching for concise representations of policies is region based policies where the policy picks a single action per region. For example, with the spiral gridworld from Example~\ref{ex:spiral}, it is sufficient to pick one action per region to navigate the agent from the initial state to the target region. However, this is not the case in general, as shown in the following example.

\begin{example}[Double pass triangle]\label{ex:double-pass-triangle}
\begin{figure}
    \centering
    \begin{subfigure}[b]{0.45\linewidth}
	    \centering
	    \includegraphics[width=\linewidth]{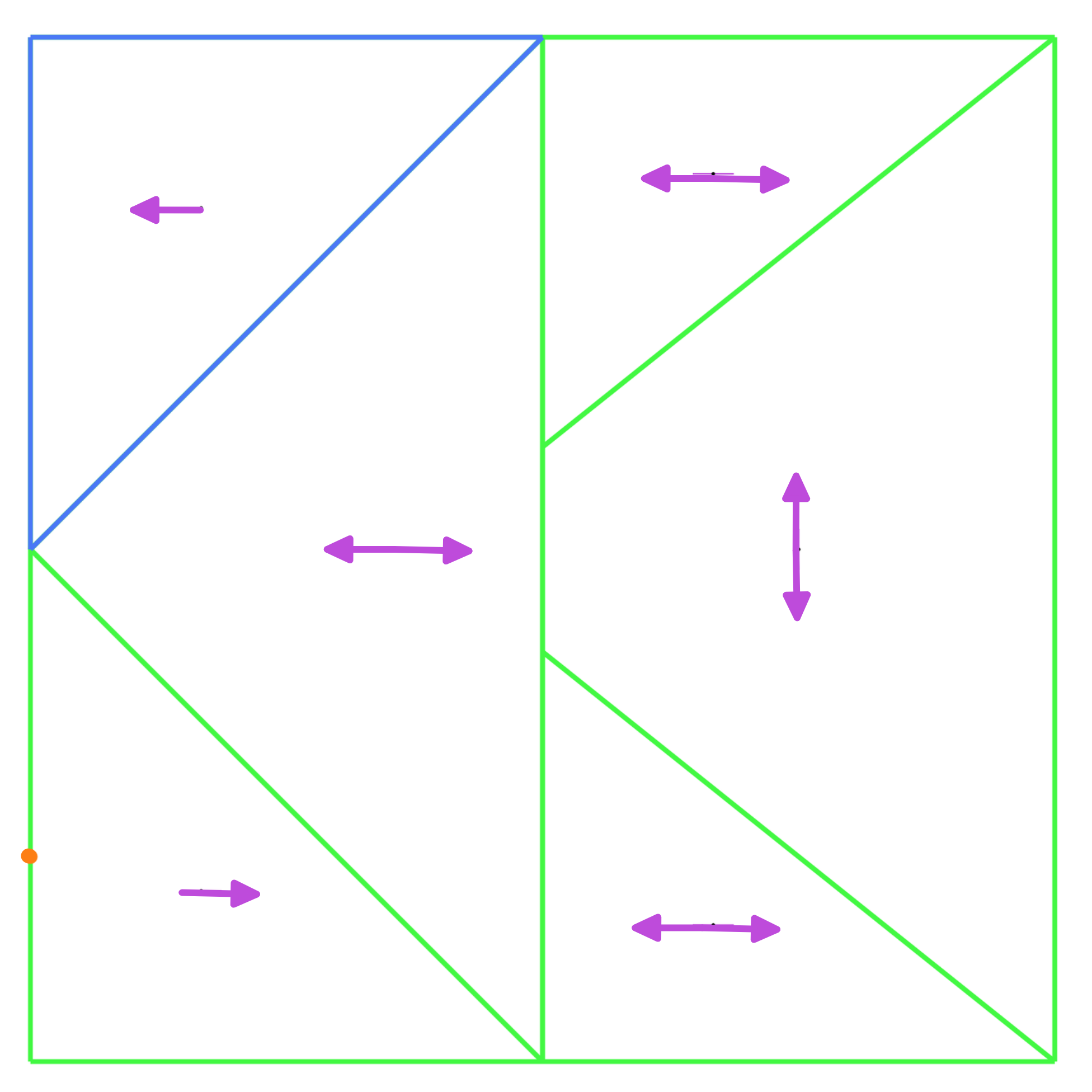}
	    \caption{Gridworld}
	    \label{fig:one-triangle-two-pass-gridworld}
    \end{subfigure}
    \begin{subfigure}[b]{0.45\linewidth}
        \centering
        \includegraphics[width=\linewidth]{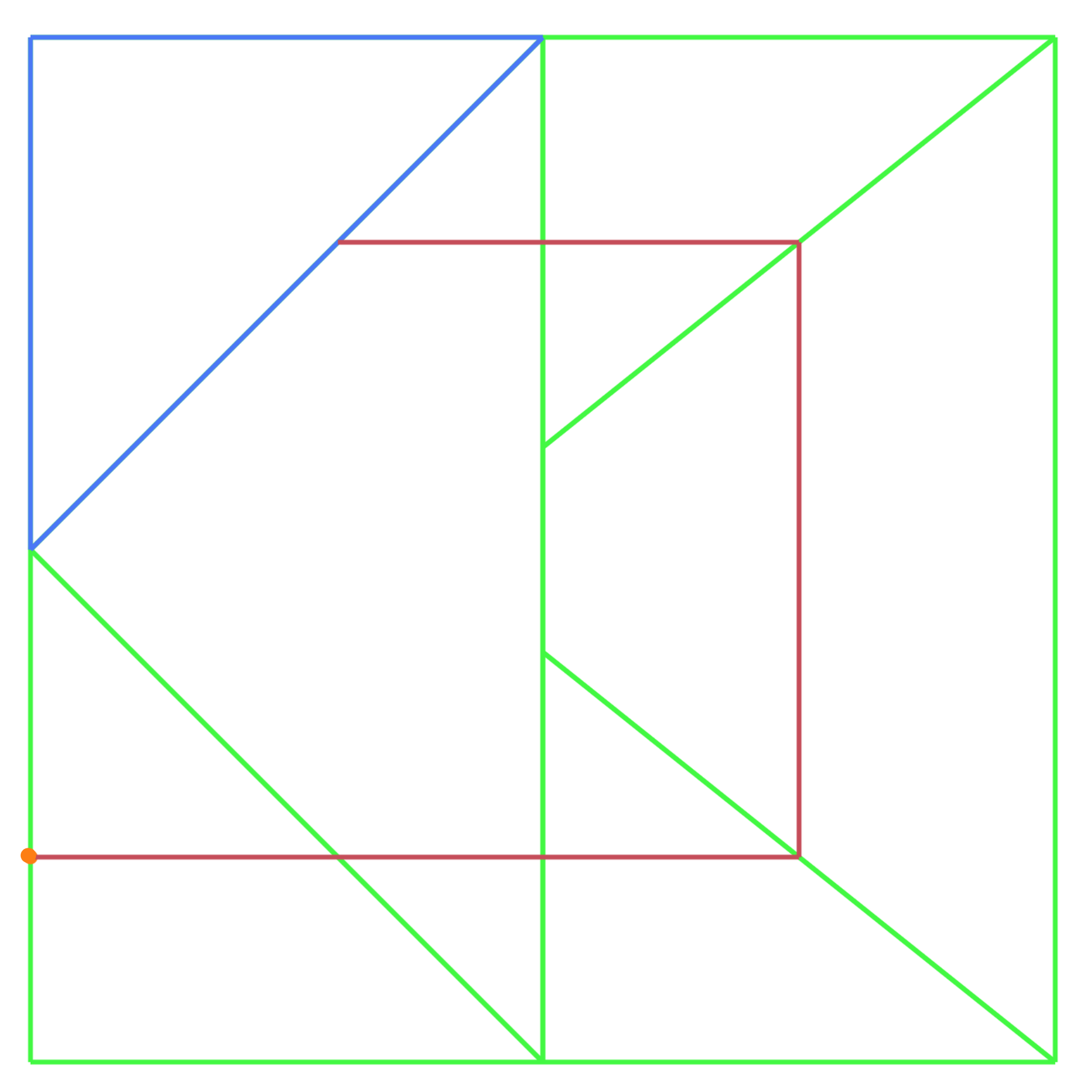}
        \caption{Policy path}
        \label{fig:one-triangle-two-pass-policy}
    \end{subfigure}
    \caption{Double pass triangle}
    \label{fig:one-triangle-two-pass}
\end{figure}
The gridworld is represented in Figure \ref{fig:one-triangle-two-pass-gridworld}.
Figure~\ref{fig:one-triangle-two-pass-policy} visualizes the (shortest) path from the initial state to the target region. 
The important remark here is that the triangular region in the middle is crossed twice, with different actions: first time right, second time left.
\end{example}

\section{The subgoal DSL}
\label{sec:shannon}
\subsection{Inspiration}

In 1950, Claude Shannon built, as a small project at home, one of the first instances of machine learning that the world had witnessed: a mechanical mouse capable of learning to solve a configurable maze in which the maze walls could be positioned as desired\footnote{\url{https://www.technologyreview.com/2018/12/19/138508/mighty-mouse/}}. To do this, he repurposed telephone relay circuits and placed them underneath the maze board to navigate the mouse towards the exit. In a first pass, the mouse would systematically explore the whole maze looking for the exit and \textit{learn} the path, so that in its subsequent attempts, it could swiftly reach the target. The magic was hidden in the relay circuits which would remember the path and were able to tell the mouse to turn left or right based on whether a switch was on or off. The first attempt is reminiscent of reinforcement learning with a trial and error approach to learn a policy to solve the game. However, our focus in this work will be on the subsequent attempts where we observe a programmatic abstraction to obtain a concise representation of the policy: instead of specifying the direction to follow at each point of the maze, the relay switches only indicated the points at which to change direction.

\begin{figure}
    \centering
    \begin{subfigure}{0.4\linewidth}
        \includegraphics[width=\linewidth]{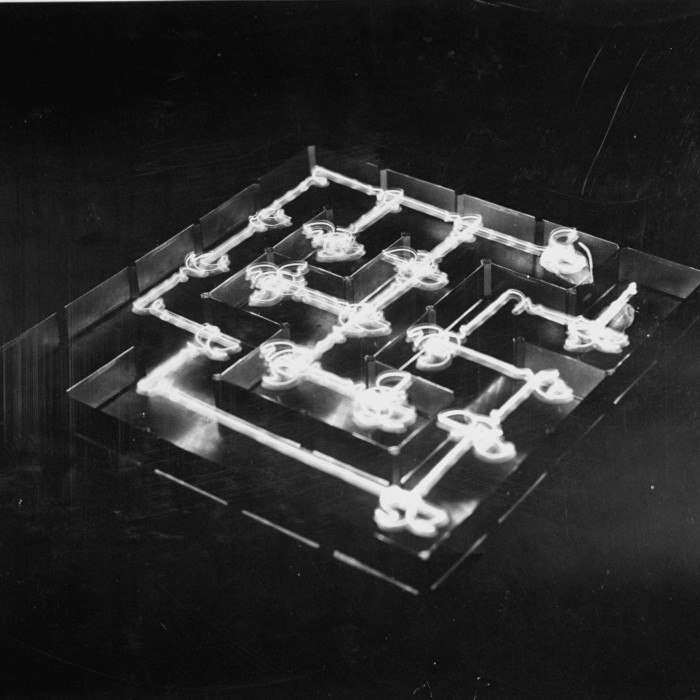}
    \end{subfigure}
    \begin{subfigure}{0.4\linewidth}
        \includegraphics[width=\linewidth]{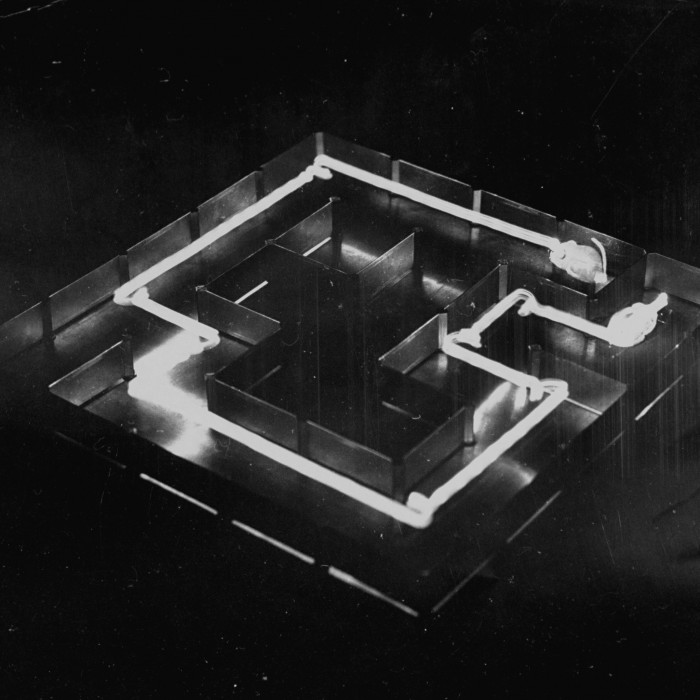}
    \end{subfigure}
    \caption{Images published in Life magazine in 1952 which show the path taken by Claude Shannon's mechanical mouse first while learning to navigate the maze and the direct path taken in its second attempt.}
    \label{fig:shannon-mouse-path}
\end{figure}

\subsection{Programmatic subgoal policies}

Before defining the DSL, let us dive into an example.

\begin{example}
The following code block contains a policy program for the spiral example given in Example~\ref{ex:spiral}.
The segments are identified by their coordinates.
\begin{lstlisting}
Do:
    From [(0, 0), (13, 13)] ->
        Target [(13, 13), (14, 12)], Preference: (13,13)
        Else Target [(14, 12), (26, 0)], Preference: (14,12)
    From [(14, 12), (26, 0)] ->
        Target [(14, 14), (28, 28)], Preference: (14, 14)
    From [(14, 14), (28, 28)] ->
        Target [(14, 14), (0, 28)], Preference: (0,28)
    From [(14, 14), (0, 28)] ->
        Target [(0, 0), (13, 13)], Preference: (0,0)
Until([(13, 13), (14, 12)])
\end{lstlisting}
This program ensures to reach the segment \([(13,13), (14,12)]\), which is an edge of the target region.
To do this, it specifies moves starting from four segments: \([(0, 0), (13, 13)]\), \([(14, 12), (26, 0)]\), \([(14, 14), (28, 28)]\), and \([(14, 14), (0, 28)]\).
Figure \ref{fig:spiral-gridworld} can help visualize these segments: this is the four diagonals.
From \([(0,0), (13,13)]\) we have two target segments: if we can reach \([(13,13), (14,12)]\), then we do (indeed this is part of the target region so we are done), otherwise we aim at \([(14,12), (26,0)]\). Importantly, when aiming at \([(14,12), (26,0)]\) we want to go as close as possible to \((14,12)\) as specified by the preference. Here this means choosing only action \texttt{RIGHT}, and not \texttt{DOWN}.
In the other three segments there is a single target segment.
\end{example}

In general, a policy program is a sequence of \texttt{Do Until} loops:
\begin{lstlisting}
Do:
    P
Until(e)
\end{lstlisting}
This is interpreted as: run the local program $P$ until reaching the edge $e$, which we call the local goal.
A local program is a set of instructions of the form
\begin{lstlisting}
From s ->
	Target s1, Preference: v1
	Else Target s2, Preference: v2
	...
	Else Target sk, Preference: vk
\end{lstlisting}
where $s,s_1,\dots,s_k$ are segments and $v_1,\dots,v_k$ states, with $v_i$ an extremal point of $s_i$.
It is interpreted as: from any state $v$ in the segment $s$, let $i$ be the \textit{least} index such that there is a move from $v$ to $s_i$;
move to the reachable state of $s_i$ closest to $v_i$. 

\section{The tree of shortest paths}
\label{sec:tree}

\subsection{The backward algorithm}
\label{subsec:backward_algorithm}
We introduce an algorithm computing the set of winning states, meaning for which there exists a path to the target region. 
At a high-level, the algorithm is a generic backward breadth-first search algorithm: starting from the target region, it builds and expands a tree where the nodes represent states that can reach the target.

The backward algorithm builds a tree as follows. The root is a special node, whose children are all the edges of the target region.
Nodes are pairs consisting of a segment $[v_1,v_2]$ and a region $R$ such that $[v_1,v_2] \subseteq R$.
To expand a node $([v_1,v_2], R)$, we identify the region $R' = \adj(R, [v_1,v_2])$ sharing $[v_1,v_2]$ with $R$, if it exists.
We then consider the set of states of $R'$ for which there exists a move to a state in $[v_1,v_2]$: it is the convex combination of segments included in the edges of $R'$.
For each such segment $[v'_1,v'_2]$, we remove from it all segments already appearing in a node of the tree, and if the segment $[v''_1,v''_2]$ it yields is non-empty, then we add a node $([v''_1,v''_2], R')$ as a child of the node $([v_1,v_2], R)$.

\begin{example}\label{ex:backward-tree-construction}
    \begin{figure}
        \centering
        \begin{subfigure}[b]{0.45\linewidth}
            \includegraphics[width=\linewidth]{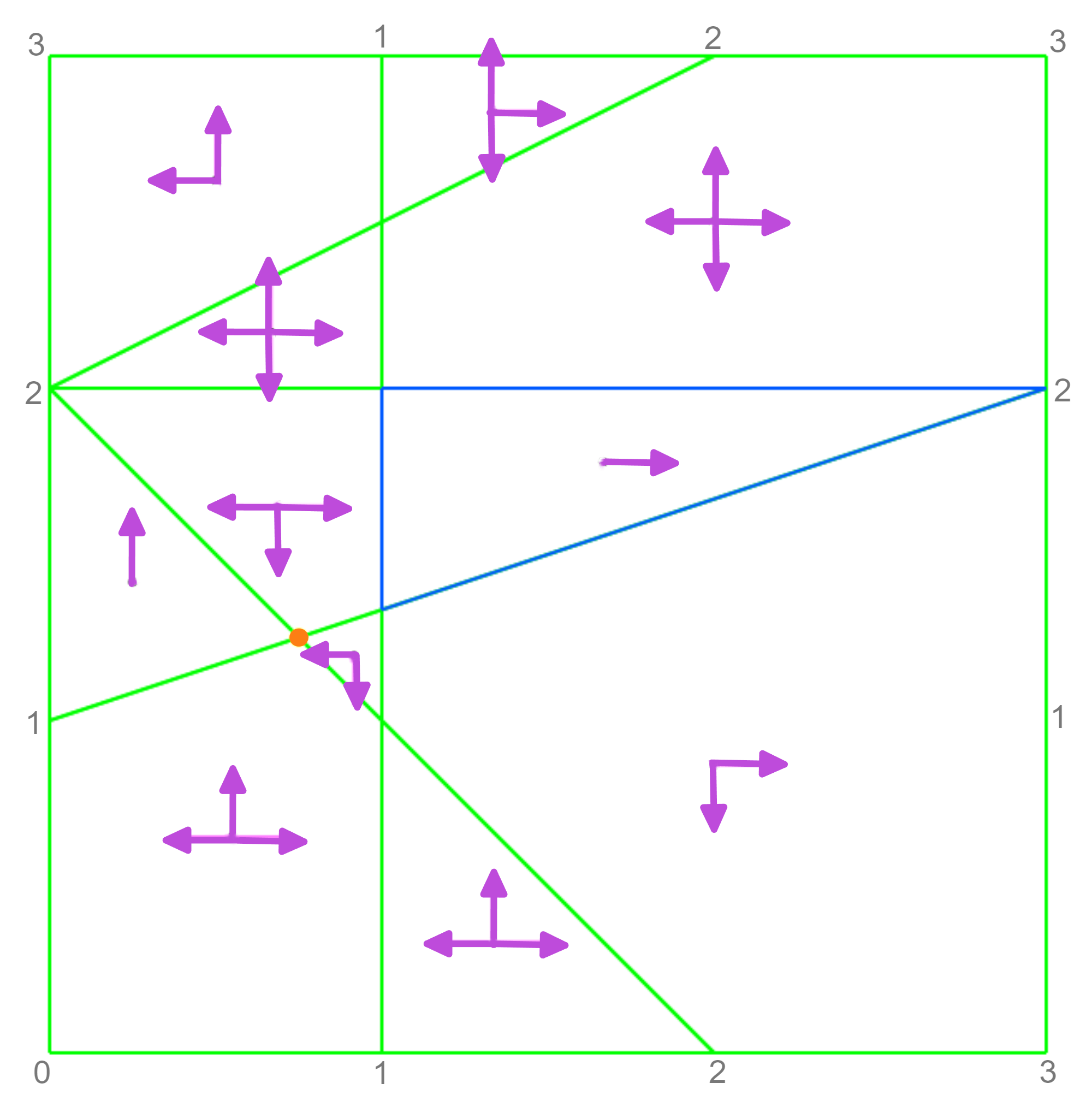}
            \caption{A gridworld}
            \label{fig:s3p5}
        \end{subfigure}
        \begin{subfigure}[b]{0.45\linewidth}
            \includegraphics[width=\linewidth]{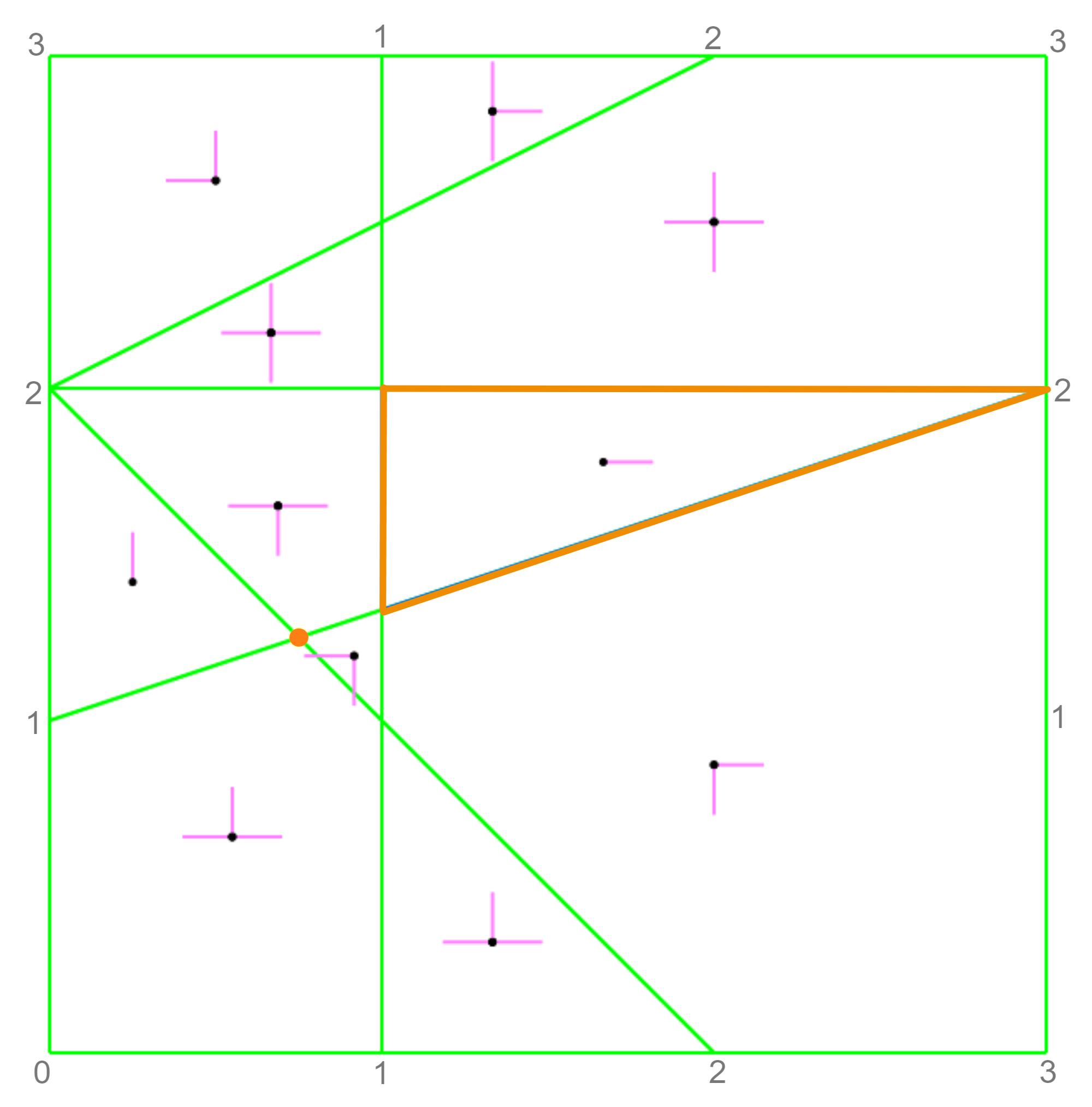}
            \caption{Depth 1}
            \label{fig:s3p5-backward-tree-0}
        \end{subfigure}
        \begin{subfigure}[b]{0.45\linewidth}
            \includegraphics[width=\linewidth]{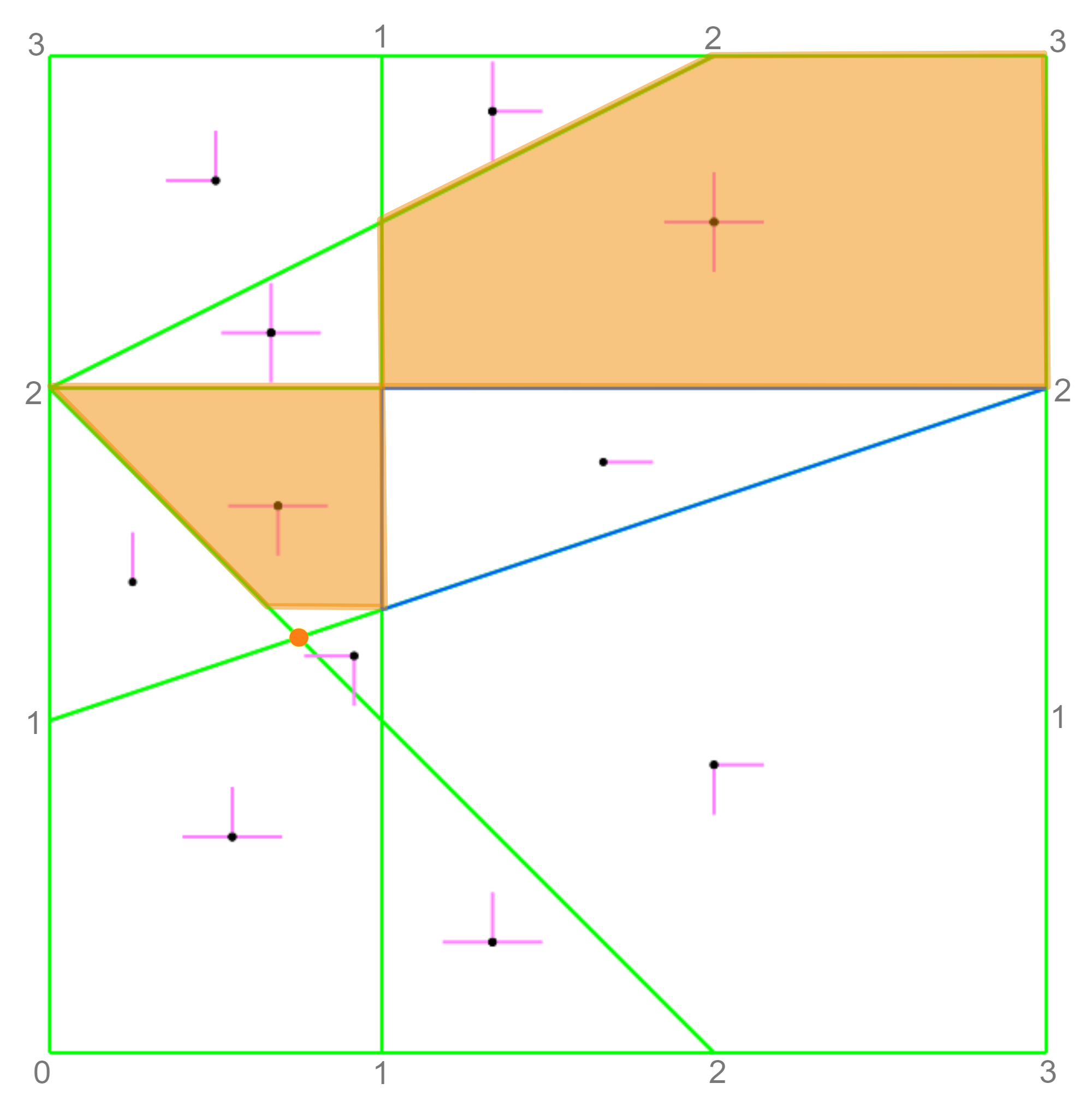}
            \caption{Depth 2}
            \label{fig:s3p5-backward-tree-1}
        \end{subfigure}
        \begin{subfigure}[b]{0.45\linewidth}
            \includegraphics[width=\linewidth]{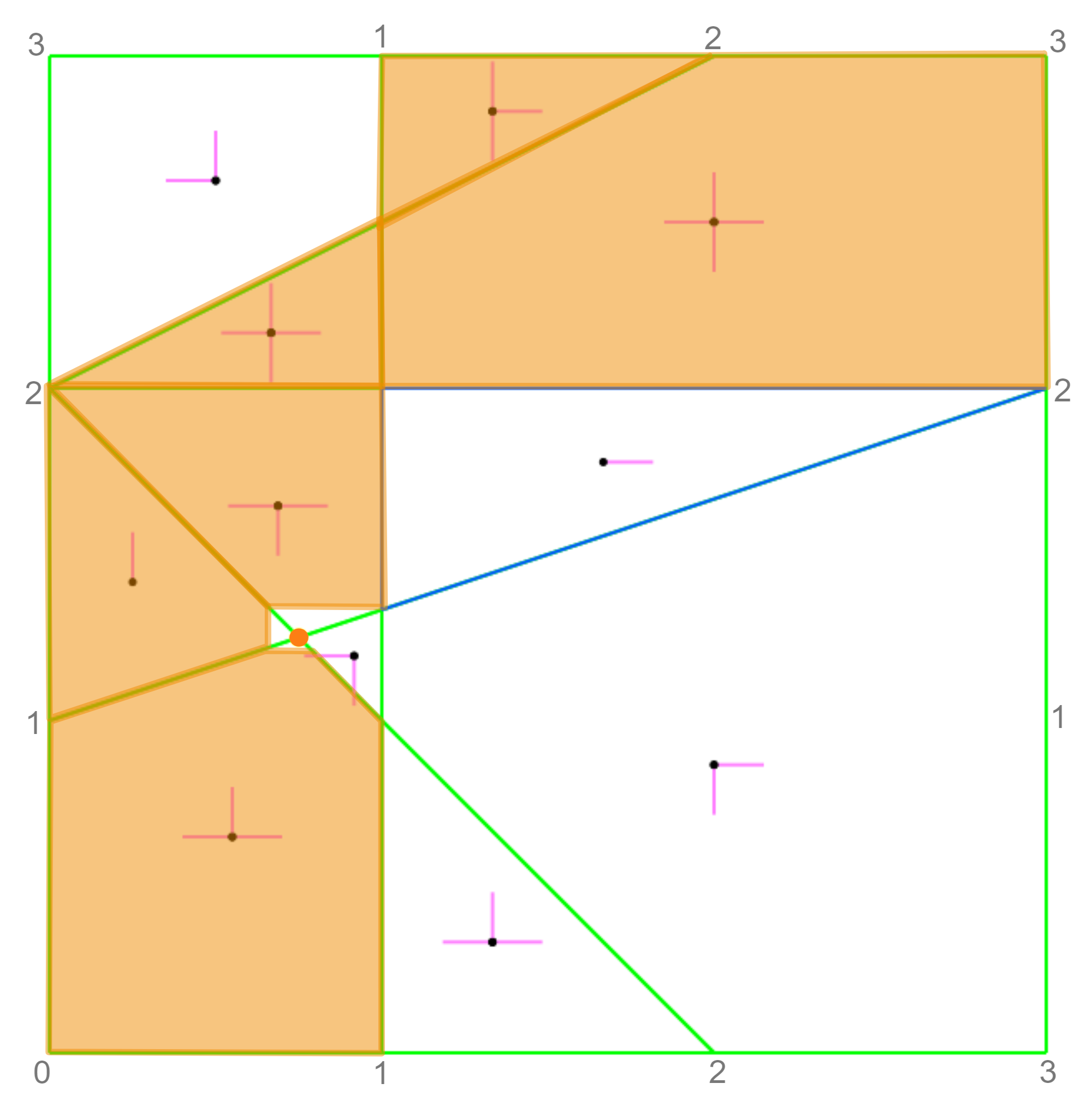}
            \caption{Depth 4}
            \label{fig:s3p5-backward-tree-3}
        \end{subfigure}
        \caption{Backward construction of the tree of the winning region}
        \label{fig:backward-tree-contruction}
    \end{figure}

In Figure \ref{fig:backward-tree-contruction}, we can visualize how the backward algorithm works. The tree itself can be seen in the appendix (Figure \ref{fig:trees-winning-region}).
\end{example}

By construction, from each state in a segment of a node of the tree we can construct a path to the target region. 
This is a shortest path, where shortest means that it is minimal in number of moves.
Conversely, every state on an edge for which there exists a path to the target region belongs to some node of the tree.
Since there might be more than one state in its parent segment that is reachable, the tree represents a family of paths to the target region. 
Importantly, these paths never visit the same point twice: they are \textit{non self-intersecting}. This is due to how we have constructed the tree, filtering out parts of segments which we have already visited.

The algorithm does not terminate in general: in this example, the tree is infinite, it contains states closer and closer to the initial state but never that state. One of the main results of~\cite{AsarinMP95} is a termination argument, giving bounds on how large the tree can be to include a fixed initial state. This implies the decidability of our problem, but does not address the question of representation of paths.
Interestingly, this decidability result is complemented by an undecidability result for the same problem in dimension 3.


\subsection{Properties of the tree}
\label{subsec:region_based}
We now prove properties of paths in the tree.
These are preliminary steps before constructing an algorithm deriving a programmatic policy from the tree.
A branch of the tree is a sequence of segments 
\[
([a_1, b_1], [a_2, b_2], \dots, [a_p, b_p]).
\]
It induces a sequence of pairs of regions and edges \(((R_1, e_1), \dots, (R_{p-1}, e_{p-1}))\) satisfying the following properties, for each \(i \in [1,p-1]\): 
(i) there exists a move between \([a_i, b_i]\) and \([a_{i+1}, b_{i+1}]\) contained in \(R_i\), and
(ii) \(R_i \neq R_{i+1}\), and
(iii) \([a_i, b_i] \subseteq e_i\), and
(iv) \(e_{i+1} \in \edges(R_i) \cap \edges(R_{i+1})\).


\begin{lemma}\label{lemma:three-time-visit}
    There exist at most two indices \(i < j \in [1,p-1]\) such that $R_i = R_j$ and $e_i = e_{i+1} = e_j = e_{j+1}$.
\end{lemma}
\begin{proof}
    Arguing by contradiction, suppose there exist 3 indices \(k < i < j\) such that 
    $R_k = R_i = R_j$ and $e_k = e_{k+1} = e_i = e_{i+1} = e_j = e_{j+1}$. 
    Let us denote the common edge by the segment \([a,b]\) and without loss of generality, assume for each of the segments \([a_i, b_i]\), \([a_{i+1}, b_{i+1}]\), \([a_j, b_j]\), \([a_{j+1}, b_{j+1}]\), \([a_k, b_k]\), and \([a_{k+1}, b_{k+1}]\), the first vertex of the segment is closer to \(a\) and the second vertex is closer to \(b\). This orientation makes the following arguments easier.

    In the rest of the proof, we base our arguments on the algorithm used to construct the tree of the winning region. First, let us place ourselves in the situation when the leaf associated to the segment \([a_{j+1}, b_{j+1}]\) was being extended. \([a_j, b_j]\) is a segment from which there exists a path to \([a_{j+1}, b_{j+1}]\). As we filter out parts of edges which have already been explored, either \([a_j, b_j] \subseteq [a, a_{j+1}]\) or \([a_j, b_j] \subseteq [b_{j+1}, b]\). We consider the first case \([a_j, b_j] \subseteq [a, a_{j+1}]\). Then, in fact we have a path from each point in \([a, a_{j+1}]\) to \([a_{j+1}, b_{j+1}]\). This is due to the fact that there exists \(x \in [a_{j+1}, b_{j+1}]\) such that \(x - a_j\) is included in the cone of available actions in \(R_j\). As for each \(y \in [a, a_{j+1}]\), \(x-y\) is along the same direction as \(x-a_j\), it is also in the cone of available actions. Thus, the whole segment \([a, b_{j+1}]\) has been explored until now.

    The next time we visit this edge, we have that \([a_{i+1}, b_{i+1}] \subseteq [b_{j+1}, b]\). Similar to before, since \([a, b_{j+1}]\) has been explored we have either \([a_i, b_i] \subseteq [b_{j+1}, a_{i+1}]\) or \([a_i, b_i] \subseteq [b_{i+1}, b]\). We consider the former case. So \([a, b_{i+1}]\) has been explored and thus \([a_{k+1}, b_{k+1}] \subseteq [b_{i+1}, b]\). Now, one can see in Figure \ref{fig:three-times-edge-visit-lemma} that any path represented by the segments which visits \([a_{k+1}, b_{k+1}] \to\) \([a_i, b_i]\to\) \([a_{i+1}, b_{i+1}]\to\) \([a_j, b_j] \to\) \([a_{j+1}, b_{j+1}] \) is necessarily self-intersecting. We have completed the case \([a_j, b_j] \subseteq [a, a_{j+1}]\) and \([a_i, b_i] \subseteq [b_{j+1}, a_{i+1}]\). Each of the three other cases can be verified similarly, they all give us self-intersecting paths.
\end{proof}

\begin{figure}
    \centering
    \includegraphics[width=0.5\textwidth]{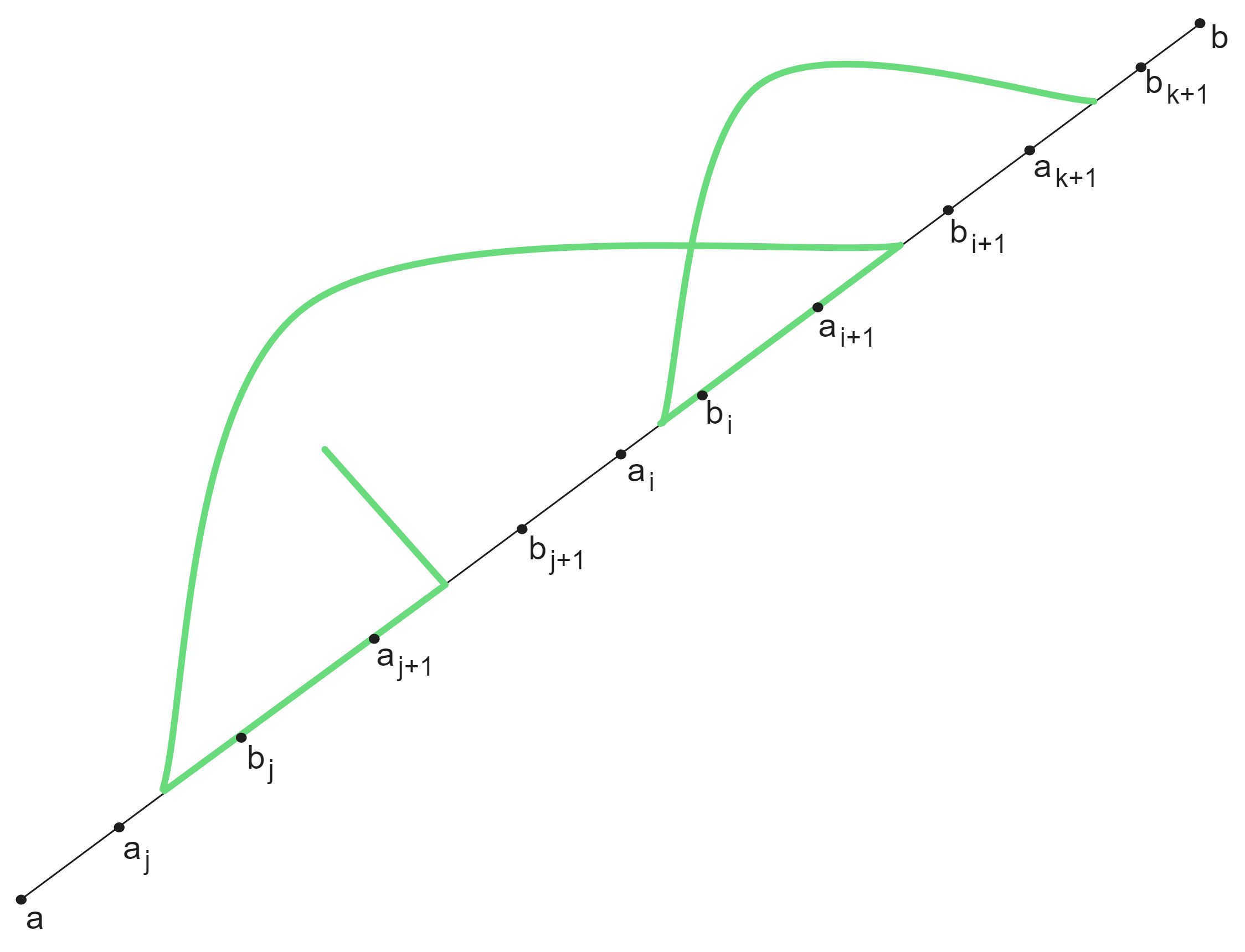}
    \caption{Any path which visits \([a_{k+1}, b_{k+1}] \to\) \([a_i, b_i]\to\) \([a_{i+1}, b_{i+1}]\to\) \([a_j, b_j] \to\) \([a_{j+1}, b_{j+1}] \) is necessarily self-intersecting.}
    \label{fig:three-times-edge-visit-lemma}
\end{figure}

\begin{lemma}\label{lemma:edges-two-ortho-dirs}
    Assume there exist indices \(i < j \in [1,p-1]\) such that \((R_i, e_i) = (R_j, e_j)\), \(e_i \neq e_{i+1}\) and \(\actions(R_i)\) contains at least two orthogonal directions. Then,
    \begin{enumerate}
        \item if the sequence of segments forms an \textbf{inner loop} at \((R_j, e_j)\), then all the edges of regions \textbf{inside the loop} are not visited by the subsequence \((e_1, \dots, e_j)\).
        \item if the sequence of segments forms an \textbf{outer loop} at \((R_j, e_j)\), then all the edges of regions \textbf{outside the loop} are not visited by the subsequence \((e_1, \dots, e_j)\).
        \item if \(j\) is the least index such that \(i < j\) and \((R_j, e_j) = (R_i, e_i)\), we can construct \([a'_{j+1}, b'_{j+1}]\) from \([a_{j}, b_{j}], [a_i, b_i]\) and \([a_{i+1}, b_{i+1}]\) such that \([a'_{j+1}, b'_{j+1}] \subseteq [a_{j+1}, b_{j+1}]\) and each point in \([a'_{j+1}, b'_{j+1}]\) is reachable from a point in \([a_j, b_j]\).
    \end{enumerate}
\end{lemma}

\begin{figure*}[!ht]
    \centering
    \begin{subfigure}[b]{0.49\textwidth}
        \centering
        \includegraphics[width=\textwidth]{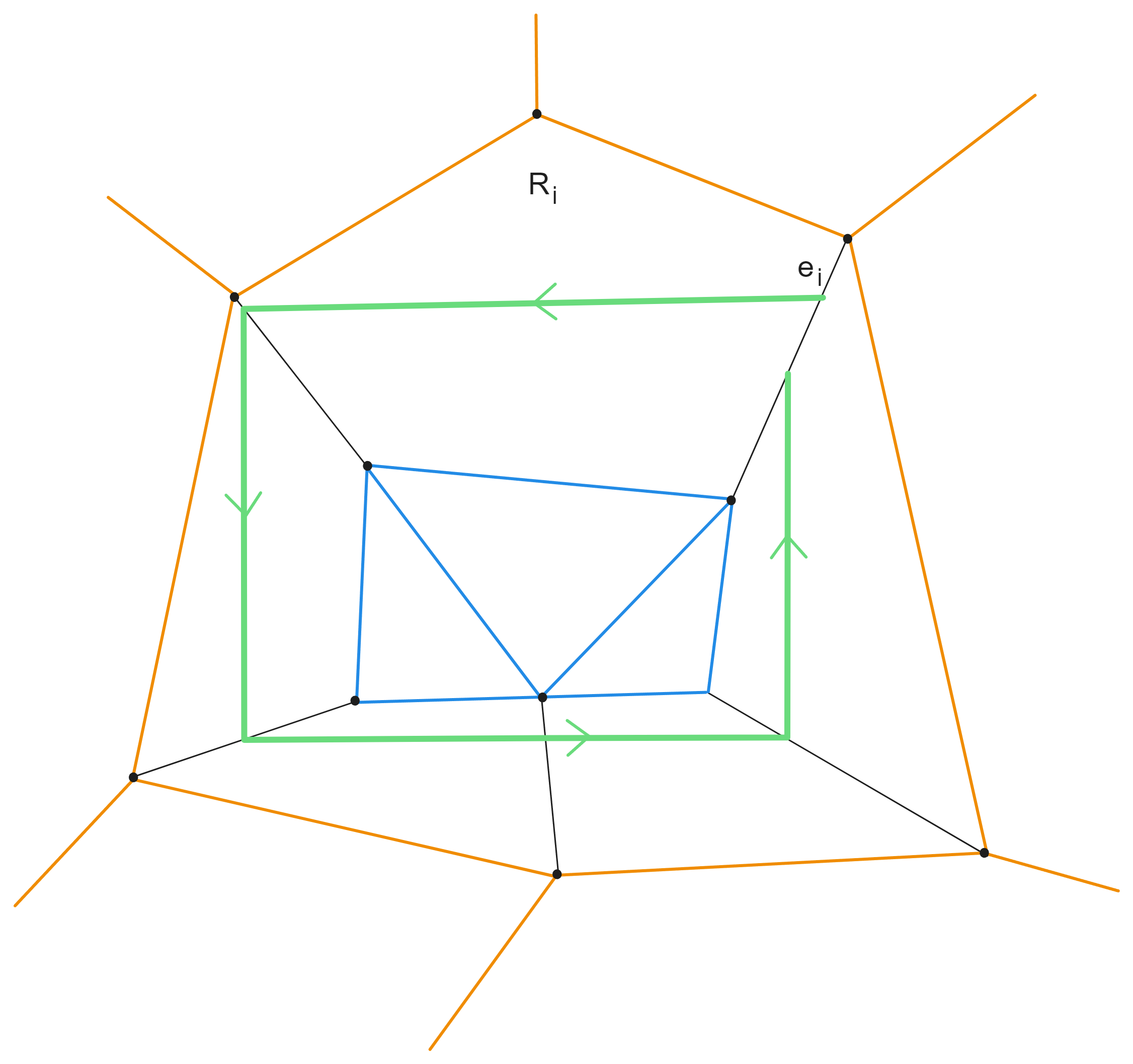}
        \caption{}
        \label{fig:loop-partition-edges}
    \end{subfigure}
    \begin{subfigure}[b]{0.49\textwidth}
        \centering
        \includegraphics[width=\textwidth]{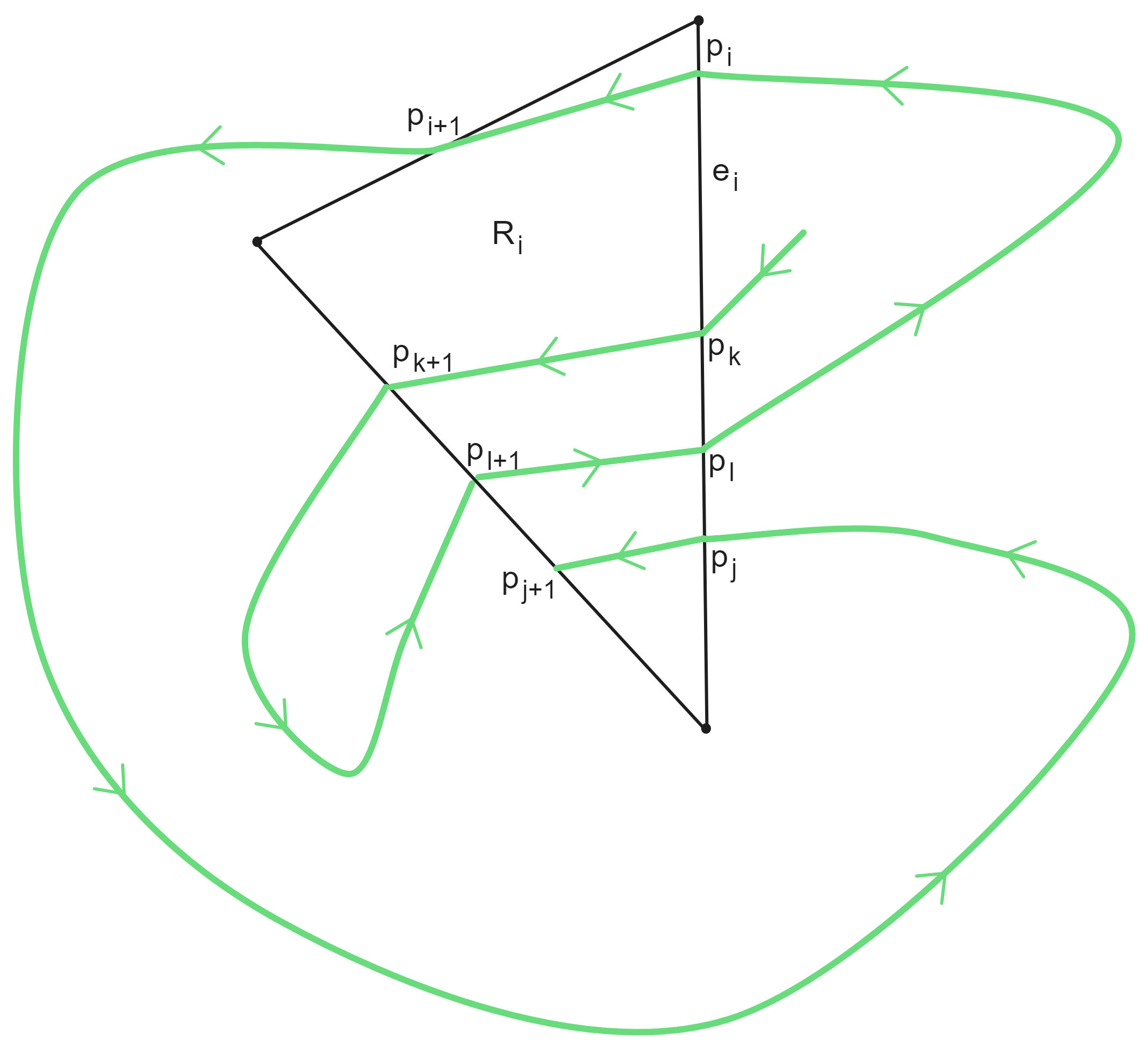}
        \caption{}
        \label{fig:loop-cannot-visit-inside}
    \end{subfigure}
    \begin{subfigure}[b]{0.49\textwidth}
        \centering
        \includegraphics[width=\textwidth]{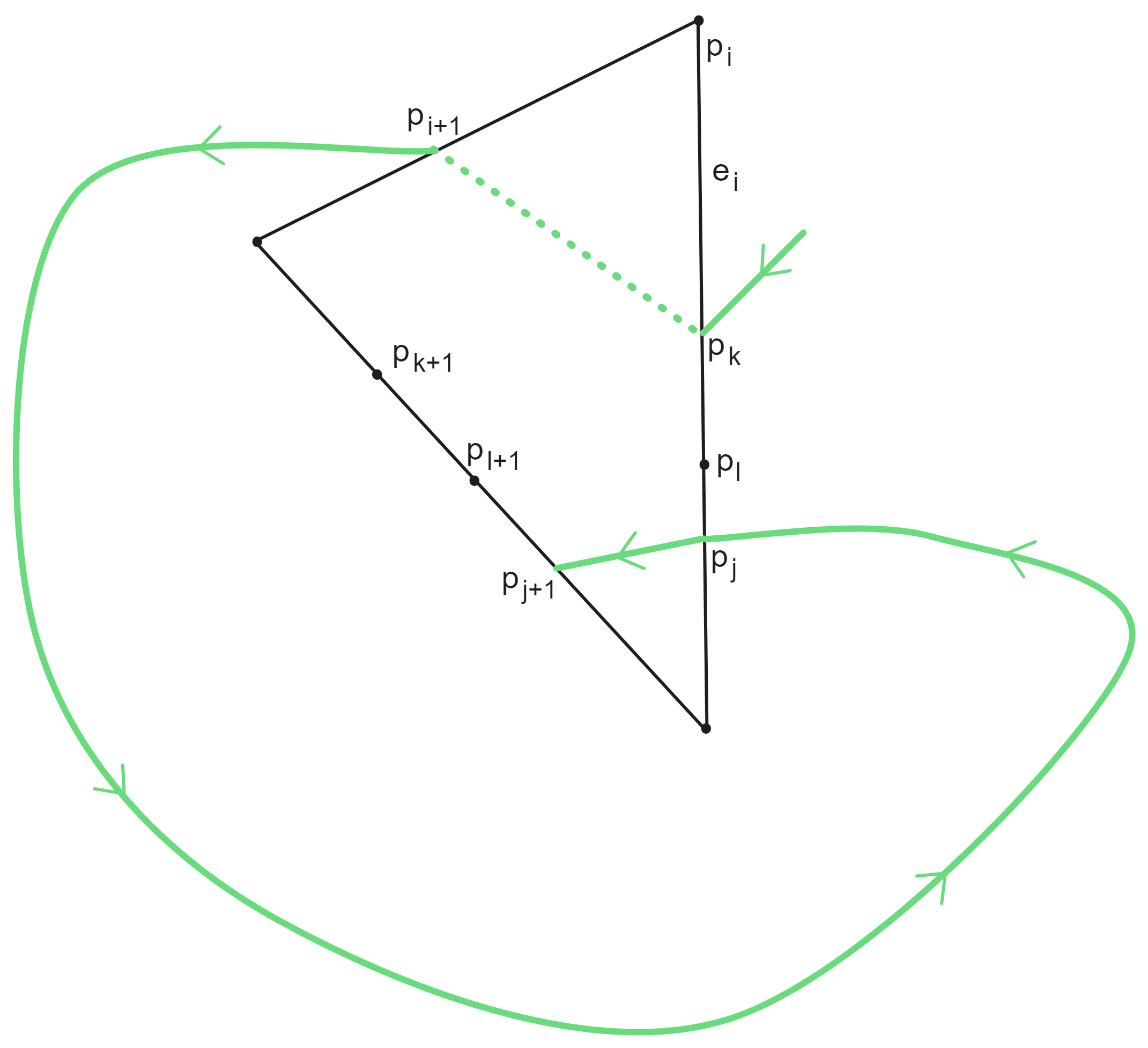}
        \caption{}
        \label{fig:loop-cannot-visit-inside-up-act}
    \end{subfigure}
    \begin{subfigure}[b]{0.25\textwidth}
        \centering
        \includegraphics[width=\textwidth]{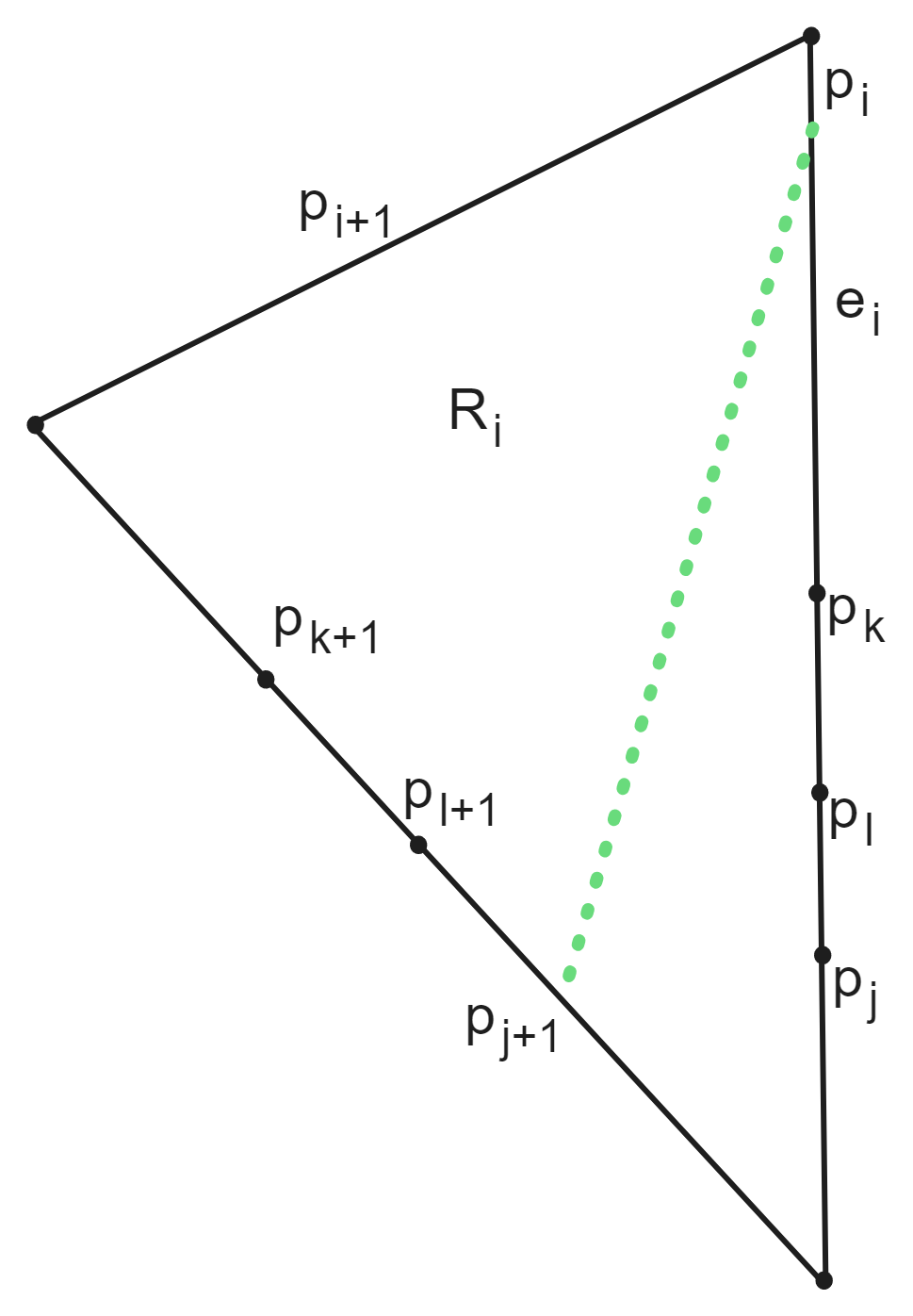}
        \caption{}
        \label{fig:loop-cannot-visit-inside-down-act}
    \end{subfigure}
    \caption{(a) Inner loop at \((R_i, e_i)\) partitions all the edges not involved in the loop into inner (blue) and outer (orange) edges (b) Loop at \((R_i, e_i)\) with the path previously visiting an inner edge at index \(k+1\) (c) Counterexample path in the case where \(\downact\) and \(\leftact\) are available in \(R_i\) (d) Counterexample path in the case where \(\upact\) and \(\leftact\) are available in \(R_i\).}
\end{figure*}

\begin{proof}
    Let us begin by understanding what we mean by a loop and edges being inside and outside loops. Looking at Figure \ref{fig:loop-partition-edges}, the path (in green) starting at \((R_i, e_i)\) forms a loop at this edge. As soon as this loop is completed, the edges of regions not involved in the loop are partitioned into disjoint two sets: the edges inside the loop (in blue) and those outside (in orange). It is important that \(e_{i+1} \neq e_i\) else the loop of edges would not be formed.

    We will now prove item 1, the proof of item 2 follows a similar pattern. To this end, we again proceed by contradiction and assume that we have a sequence of segments which visits an inner edge before forming a loop at \((R_j, e_j) = (R_i, e_i)\). Let \(p_i \in [a_i, b_i], p_j \in [a_j, b_j]\). There exists a path going from \(p_i\) to \(p_j\). Now take a look at Figure \ref{fig:loop-cannot-visit-inside} in which we can see a loop being formed by the path from \(p_i\) to \(p_j\). If this sequence of segments visits an inner edge before forming the loop, it has to pass through the space in the edge between \(p_i\) and \(p_j\) because if not, we would have a self-intersecting path. Thus, there exist indices \(k < l < i\) such that \(e_k = e_l = e_i\) which are the indices where the sequence enters and exits the edge \(e_i\) while visiting an inner edge. Let \(p_k \in [a_k, b_k], p_l \in [a_l, b_l]\) and consider a path visiting \(p_k \to\) \(p_l \to \) \(p_i \to\) \(p_j\) which can also be seen in Figure \ref{fig:loop-cannot-visit-inside}.

    Since \(\actions(R_i) \neq \emptyset\), without loss of generality, we assume that \(\leftact \in \actions(R_i)\) as seen in the figure. Using the assumption that we have at least two orthogonal directions available in \(R_i\), we have two cases: \(\upact \in \actions(R_i)\) or \(\downact \in \actions(R_i)\). In the first case, \(p_{i+1}\) would be reachable from \(p_k\) and \([a_k, b_k]\) (see Figure \ref{fig:loop-cannot-visit-inside-up-act}) and so this segment would have been explored while the node \([a_{i+1}, b_{i+1}]\) was being extended which means that \([a_k, b_k]\) cannot exist in the space between \(p_i\) and \(p_j\). Similarly in the second case, if \(\downact \in \actions(R_i)\), \(p_{j+1}\) would be reachable from \(p_i\)  and \([a_i, b_i]\) (see Figure \ref{fig:loop-cannot-visit-inside-down-act}) and therefore for the same reasons, \([a_i, b_i]\) cannot exist above \(p_j\). This concludes the proof of item 1.
\end{proof}

\section{Constructing programmatic policies}
\label{sec:deriving_policy_programs}
What remains to be done is derive from the tree a programmatic policy; one could say ``compress'' paths, or find regularity. This is the purpose of Algorithm~\ref{alg:synth-policy}. Lemma \ref{lemma:edges-two-ortho-dirs} provides the main idea for the construction of the policy synthesis procedure. As we have shown, whenever a loop is formed by the path, a new edge is discovered as soon as the loop is exited due to the non self-intersecting property. This motivates us to consider programmatic policies with a sequence of \texttt{Do Until} blocks corresponding to the sequence of edges in the order in which they are discovered by a path.


The algorithm works as follows: it takes in a sequence of segments from the tree \(([a_1, b_1], \dots, [a_p, b_p])\) and the corresponding sequence of edges \((e_1, \dots, e_p)\). It goes through both these sequences and each time a new edge is encountered, it begins a new \texttt{Do Until} block. At each iteration of the loop, if it sees that a segment of the next edge is already a target, i.e., if it visits the same pair of consecutive edges twice, it merges the two segments. This merging procedure ensures that when we encounter a loop in our path, segments belonging to the same edge are merged thus resulting in a compact representation of the sequence of segments. When we merge, we also set the preference depending on which side the next segment is with respect to the previous segment on the same edge. Intuitively, this allows to distinguish between \textit{inner} and \textit{outer} loops where the preference would force the policy to navigate towards a certain extreme of a segment thereby allowing the agent to progress closer towards the target region. The regions in which the allowed actions are a subset of \(\{\leftact, \rightact\}\) or \(\{\upact, \downact\}\) are handled differently by the algorithm. Since diagonal directions are not allowed in such regions, it suffices to specify the direction in which to navigate.

\begin{algorithm}[t]
    \caption{Synthesizing a programmatic policy}\label{alg:synth-policy}
    \textbf{Input:} A branch in the tree \(([a_1, b_1], \dots, [a_p, b_p])\) and the corresponding sequence of edges \((e_1, \dots, e_p)\) and regions \((R_1, \dots, R_{p-1})\)\\
    \textsf{VisitedEdges} = $\emptyset$\\
    \For{\textnormal{\(i = 1\) \textsf{to} \(p-1\)}}{
        \uIf{\textnormal{\(e_i \notin \textsf{VisitedEdges}\)}}{
            add \(e_i\) to \textsf{VisitedEdges}\\
            start new \texttt{Do Until} block with local goal $e_i$\\
        }
        \uIf{\textnormal{there is no \texttt{From \(e_i\)} in current \texttt{Do Until} block}}{
            add a \texttt{From \(e_i\)} instruction to the current block
        }
        \uIf{\textnormal{(\(\actions(R_i) \subseteq \{\leftact, \rightact\}\) or \(\actions(R_i) \subseteq \{\upact, \downact\}\)) and \(e_{i+1} \neq e_i\)}}{
            let \texttt{dir} \(\in \{\leftact, \rightact, \upact, \downact\}\) such that \texttt{dir} is the direction of \(e_{i+1}\) with respect to \(e_i\)\\
            add the instruction \texttt{GO dir} inside the \texttt{From \(e_i\)} instruction if it is not already present
        }
        \uElseIf{\textnormal{some segment \([a, b]\) included in \(e_{i+1}\) is currently the preferred target segment of \(e_i\) and \(e_{i+1} \neq e_i\)}}{
            merge \([a,b]\) and \([a_{i+1}, b_{i+1}]\) into one segment\\
            \uIf{\textnormal{merged segment is \([a, b_{i+1}]\)}}{
            set preference to $a$
            }
            \uIf{\textnormal{merged segment is \([a_{i+1}, b]\)}}{            
            set preference to $b$
            }
        }
        \uElse{
        add to the top of the instruction \texttt{From} $e_i$ a new target \([a_{i+1}, b_{i+1}]\) with preference $a_{i+1}$\\
        }
    }
\end{algorithm}



We now shift our attention to proving the correctness of Algorithm \ref{alg:synth-policy}, which means proving an \textit{expressivity result}, given by Theorem \ref{theorem:synthesis-correctness} and a \textit{succinctness result} in the form of an upper bound on the size of the synthesized programmatic policies, in Theorem \ref{theorem:synthesis-size}.

\begin{theorem}\label{theorem:synthesis-correctness}
    Given the shortest sequence of segments in the tree \(([a_1, b_1], \dots, [a_p, b_p])\) going from the initial state to the target region, Algorithm \ref{alg:synth-policy} synthesizes an optimal programmatic policy that can navigate an agent through these segments.
\end{theorem}

\begin{theorem}\label{theorem:synthesis-size}
    Given the shortest sequence of segments in the tree \(([a_1, b_1], \dots, [a_p, b_p])\) going from the initial state to the target region, Algorithm \ref{alg:synth-policy} synthesizes an optimal programmatic policy of size at most \(O(\abs{\texttt{Regions}}^4)\).
\end{theorem}

In Theorem~\ref{theorem:synthesis-size}, we made the assumption that each segment in the sequence can be stored in constant space. 
This would imply that we can store rationals of arbitrary precision representing the endpoints of the segments in constant space. Obviously, this is not a valid assumption in practice and in the case where all the edges of the regions are described by rationals, we prove the following upper bound on the space required to store the segments. 

\begin{lemma}\label{lemma:size-segment}
Suppose there exists \(D \in \N\) such that each of the endpoints of each of the edges of the regions are of the form $\left(\frac{a}{D}, \frac{b}{D}\right)$ for some \(a, b \in [0, D]\), then each segment of a path of the tree \(([a_1, b_1], \dots, [a_p, b_p])\) can be stored in space at most \(O(pD\log(D))\) when both \(a_p\) and \(b_p\) can be written in the same form.
\end{lemma}

\section{Conclusions and future work}
\label{sec:conclusions}
This work is a first step towards theoretical foundations of programmatic reinforcement learning, and more specifically the question of designing domain-specific languages for policies. The take away message is that for a large class of environments we were able to construct programmatic policies using in a non-trivial way control loops. We proved expressivity results, meaning existence of optimal programmatic policies, as well as succinctness results, proving that there exist optimal programmatic policies of size polynomial in the number of regions.
We hope that this paper will open a fruitful line of research on the theoretical front. We outline here promising directions.

A burning question is studying the trade-offs between sizes of programmatic policies and their performances. In this paper, we focused on optimal policies, meaning shortest paths to the target region. Are there smaller programmatic policies ensuring near optimal number of moves?

The motivations of this work is to construct programmatic policies because they are readable, interpretable, and verifiable. Hence alongside with expressivity and succinctness results, we should also investigate how we can reason with programmatic policies, and in particular verify them. Developing verification algorithms for programmatic policies is a natural next step for this work. Another desirable property is generalizability: programmatic policies are expected to generalize better, as it was argued in the original papers~\cite{BastaniPS18,Inala2020SynthesizingPP,Trivedi2021LearningTS}. Further theoretical and empirical studies will help us understand this argument better.

Last but not least, once we understand which classes of programmatic policies are expressive and succinct, remains the main question: how do we learn programmatic policies? Many approaches have been developed for decision trees, PIDs, and related classes. Learning more structured programmatic policies involving control loops is a very exciting challenge for the future, which has been tackled very recently~\cite{MoraesAFL23,AleixoL23,Batz_2024}!

\section*{Acknowledgements}

This work was supported by the SAIF project, funded by the ``France 2030'' government investment plan managed by the French National Research Agency, under the reference ANR-23-PEIA-0006.

\bibliography{refs}

\newpage
\onecolumn
\large
\appendix

\section{Illustration of the tree}
See Figure~\ref{fig:trees-winning-region}.
We represent 3 subtrees, each corresponding to an edge of the target region. Starting from these 3 edges, we add segments of edges of adjacent regions as nodes to the trees in a breadth-first manner. There exists a path from each state in a node to a state in its parent node. 

    \begin{figure*}
        \centering
        \begin{tikzpicture}
            \node[state, rectangle](root1){\([(1,4/3), (1,2)]\)};
            \node[state, rectangle, below of=root1, yshift=-80](root2){\([(1, 2), (3,2)]\)};
            \node[state, rectangle, below of=root2, yshift=-65](root3){\([(1,4/3), (3,2)]\)};
            \node[state, rectangle, right of=root1, xshift=70, yshift=25](root1d1l1){\([(0,2), (1,2)]\)};
            \node[state, rectangle, below of=root1d1l1](root1d1l2){\([(2/3,4/3), (0,2)]\)};
            \node[state, rectangle, below of=root1d1l2](root2d1l1){\([(1,2), (1,5/2)]\)};
            \node[state, rectangle, below of=root2d1l1](root2d1l2){\([(1,5/2), (2,3)]\)};
            \node[state, rectangle, below of=root2d1l2](root2d1l3){\([(2,3), (3,3)]\)};
            \node[state, rectangle, below of=root2d1l3](root2d1l4){\([(3,3), (3,2)]\)};
            \node[state, rectangle, right of=root1d1l1, xshift=70](root1d2l1){\([(0,1), (0,2)]\)};
            \node[state, rectangle, below of=root1d2l1](root1d2l2){\([(0,1), (2/3, 11/9)]\)};
            \node[state, rectangle, above of=root1d2l1](root1d2l3){\([(0,2), (1,5/2)]\)};
            \node[state, rectangle, right of=root2d1l2, xshift=70, yshift=25](root2d2l1){\([(1, 5/2), (1,3)]\)};
            \node[state, rectangle, below of=root2d2l1](root2d2l2){\([(1,3), (2,3)]\)};
            \draw
            (root1d1l1) edge (root1)
            (root1d1l2) edge (root1)
            (root2d1l1) edge (root2)
            (root2d1l2) edge (root2)
            (root2d1l3) edge (root2)
            (root2d1l4) edge (root2)
            (root1d2l1) edge (root1d1l2)
            (root1d2l2) edge (root1d1l2)
            (root1d2l3) edge (root1d1l1)
            (root2d2l1) edge (root2d1l2)
            (root2d2l2) edge (root2d1l2);
        \end{tikzpicture}
        \caption{Tree of the winning region corresponding to Figure \ref{fig:backward-tree-contruction} until depth 3. The leftmost nodes are the three children of the root. In this tree we only indicate segments, not the corresponding regions.}
        \label{fig:trees-winning-region}
    \end{figure*}
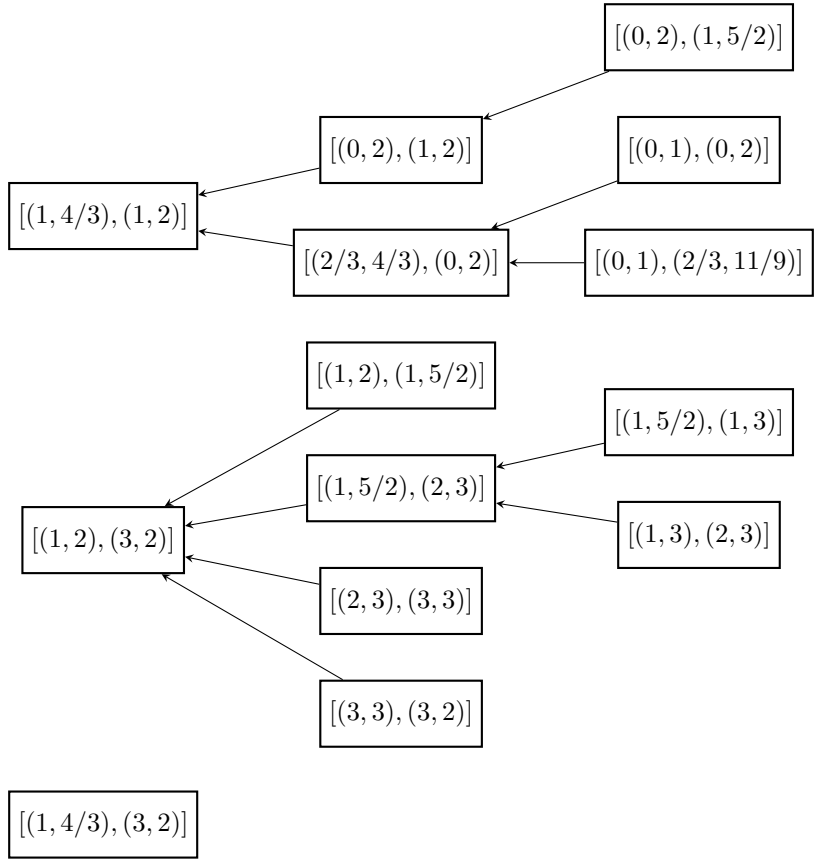

\section{Proof of the third item of Lemma~\ref{lemma:edges-two-ortho-dirs}}
We prove item 3. Let us denote \([a_z, b_z]\) by \([(x_z, y_z), (x'_z, y'_z)]\) for \(z \in \{i, i+1, j, j+1\}\). Same as before, let us assume without loss of generality that \(\leftact \in \actions(R_i)\) and \(x_{i+1} \leq x_i, x_{j+1} \leq x_j\). This can be visualized through Figure \ref{fig:loop-computation}. As at least two orthogonal directions are available in \(R_i\), let us again split into two cases with the first one being \(\upact \in \actions(R_i)\). Firstly, \([a_j, b_j]\) is below \([a_i, b_i]\) as seen in the figure because if not, \([a_{j+1}, b_{j+1}]\) (which would also have to be above \([a_{i+1}, b_{i+1}]\) to avoid self-intersecting paths) would be reachable from \([a_i, b_i]\) so it would have already been explored at index \(j\) and cannot exist there at index \(i\). By assumption, as \(j\) is the least such index satisfying the property, using arguments similar to the previous part of the proof, we have that there is no index \(k > j\) such that \([a_k, b_k] \subseteq [a_i, b_j]\). This means that when the tree node associated to the segment \([a_{j+1}, b_{j+1}]\) was being extended, \([a_i, b_j]\) was unexplored. Also, we have that \(\downact \notin \actions(R_i)\), otherwise \([a_{j+1}, b_{j+1}]\) would be reachable from \([a_i, b_i]\). Thus necessarily, \(y_{j+1} = y_j\), i.e., \(a_{j+1}\) lies on the same \(y\)-coordinate as \(a_j\). As a result, we can determine \(a_{j+1}\) simply by intersecting the line \(y = y_j\) with the region \(R_i\). Lastly, \(y'_j \leq y'_{j+1}\), i.e., \(b_j\) lies below \(b_{j+1}\) and therefore we can set \(a'_{j+1} \coloneqq a_{j+1}\) (which can be computed) and \(b'_{j+1}\) to be other point on the edge of \(R_i\) on the same \(y\)-coordinate as \(b_j\). We remark that there is a path from each point in \([a_j, b_j]\) to \([a'_{j+1}, b'_{j+1}]\): just go left! Symmetric arguments can be used to deal with the case in which \(\downact \in \actions(R_i)\).

\begin{figure}
    \centering
    \includegraphics[width=0.7\textwidth]{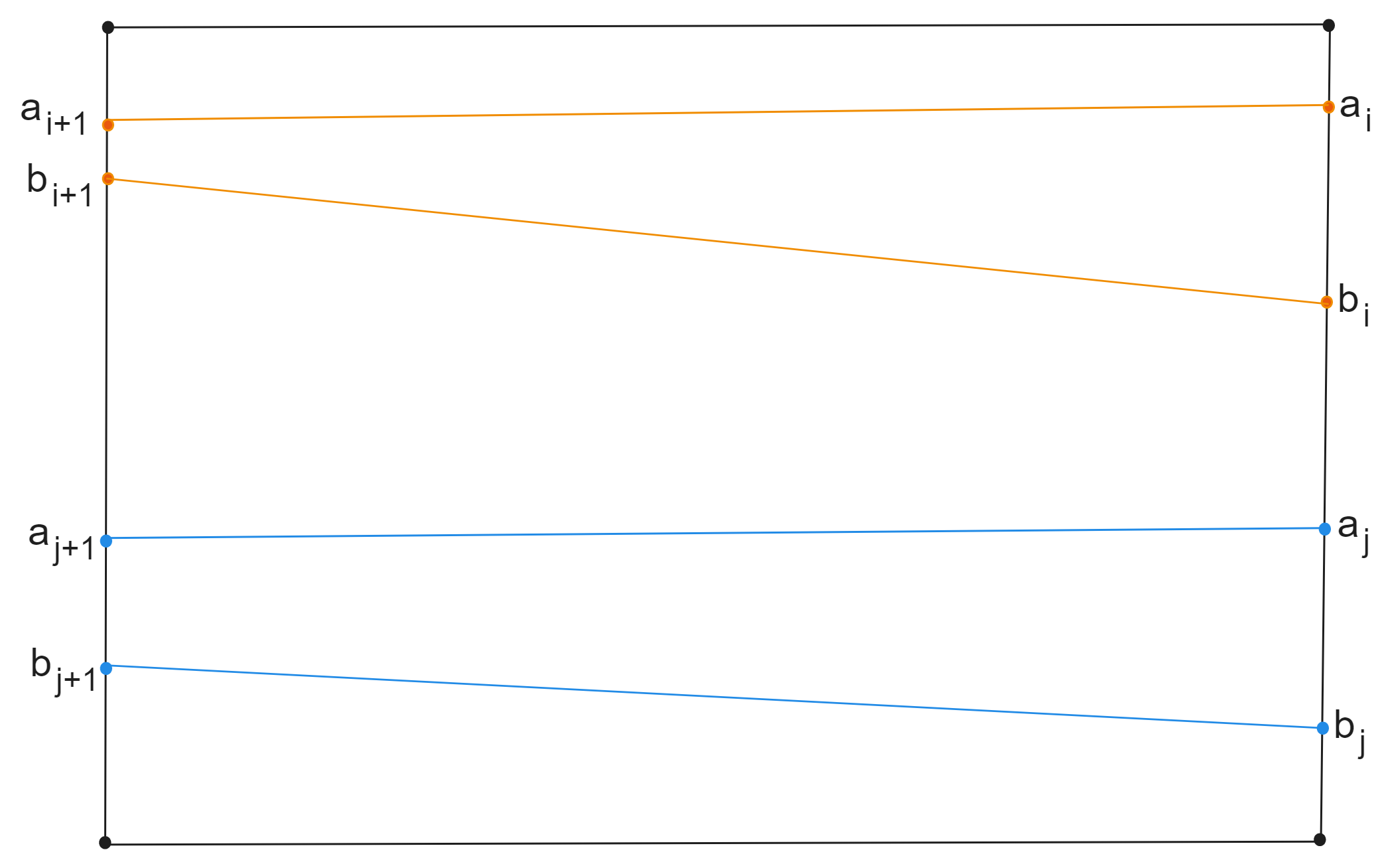}
    \caption{Constructing \([a'_{j+1}, b'_{j+1}]\) from \([a_{j}, b_{j}], [a_i, b_i]\) and \([a_{i+1}, b_{i+1}]\) in the proof of Lemma \ref{lemma:edges-two-ortho-dirs}.}
    \label{fig:loop-computation}
\end{figure}

\section{Correctness of the compression algorithm: Proof of Theorem~\ref{theorem:synthesis-correctness}}
\begin{proof}
    We will show that the synthesized policy navigates an agent through the sequence of segments in the same order. Arguing by induction, it suffices to show that if the agent is currently at a point \(p_i \in [a_i, b_i]\) for some \(i \in [1, p-1]\), the policy would guide the agent towards a point \(p_{i+1} \in [a_{i+1}, b_{i+1}]\). When the agent is at \(p_i\), the execution of the program would be in a certain \texttt{Do Until} block. Firstly, we argue that the edge \(e_i\) that contains \([a_i, b_i]\) contains at least one target in its \texttt{From} \(e_i\) instruction in the current block. This is true by construction because the synthesis algorithm processes each segment sequentially and as a result each edge that appears in the sequence would have an associated \texttt{From} instruction. Note that the only case when this would be untrue is when the goal of the \texttt{Do Until} block is reached in which case the program execution would switch to the next block.

    Next, within a \texttt{From} instruction, there may be several target segments separated by \texttt{Else} statements. Again, by construction, at least one of them is reachable from \(p_i\). Furthermore, the first reachable target segment contains \(([a_{i+1}, b_{i+1}])\) because if not, this target segment (which would appear at index greater than \(i+1\)) would be reachable from \([(a_i, b_i)]\). Consequently, we would have a shorter sequence of segments leading to a contradiction. 

    Lastly, it remains to prove that the policy would indeed navigate to a point in \(([a_{i+1}, b_{i+1}])\). Here, we need to distinguish two cases. In the first case, suppose that the target segment in the \texttt{From} \(e_i\) instruction was formed without any merging of segments. Then necessarily the target segment coincides with \(([a_{i+1}, b_{i+1}])\) or it is a region in which the allowed actions are a subset of \(\{\leftact, \rightact\}\) or \(\{\upact, \downact\}\). Both the scenarios are easily handled. The interesting case is when the target segment was formed by the merging of segments. It means that the edge \(e_i\) was visited twice and we have a loop at \(e_i\). This is where the \texttt{Preference} plays a role in navigating the agent in the correct direction. Note that at least two orthogonal actions allowed in \(\mathcal{R}_i\). As we noted in the proof of Lemma \ref{lemma:edges-two-ortho-dirs}, when we enter a loop in such a region, the segment \([a_{i+1}, b_{i+1}]\) has the same \(x\) or \(y\) coordinates as \([a_i, b_i]\) depending on the actions allowed. This means that there is a point \(p_{i+1}\) with the same \(x\) or \(y\) coordinate as \(p_i\) in the target segment. By taking another look at Figure \ref{fig:loop-computation}, we can further see that \(p_{i+1}\) coincides with the point reachable in the target segment that is extremal with respect to the \texttt{Preference}. Here, the synthesized merged segment would be \([a_{i+1}, b_{j+1}]\) with \texttt{Preference: \(b_{j+1}\)}. If the agent is at a point in \([a_j, b_j]\), and the allowed actions are \texttt{UP} and \texttt{LEFT}, the policy would navigate the agent towards a point in \([a_{j+1}, b_{j+1}]\) with the same \(y\)-coordinate (i.e., go \texttt{LEFT}) because that would be the point that is extremal with respect to the specified \texttt{Preference}.
\end{proof}

\section{Analysis of the size of programs: Proof of Theorem~\ref{theorem:synthesis-size}}
\begin{proof}
    Firstly, we remark that as we have at most \(\abs{\texttt{Regions}}^2\) edges in the gridworld, we have at most \(\abs{\texttt{Regions}}^2\) blocks \texttt{Do Until} representing the subgoals. Suppose that the sequence of segments visits \(q\) unique edges and let \(\Bar{e}_1 \to \Bar{e}_2 \to \dots \to \Bar{e}_q\) denote the sequence of these edges in the order that they are first visited. Each of these correspond to a \texttt{Do Until} block in the policy.

    Let us now analyse the size of each \texttt{Do Until} block which corresponds to a programmatic representation of a part of the sequence of segments going from \(\Bar{e}_m\) to \(\Bar{e}_{m+1}\) for a certain \(m \in [1, q-1]\). Let \(([c_1, d_1], \dots, [c_l, d_l])\) denote this sequence of segments which forms a part of the sequence \(([a_1, b_1], \dots, [a_p, b_p])\). Note that \([c_1, d_1]\) and \([c_l, d_l]\) are segments within the edges \(\Bar{e}_m\) and \(\Bar{e}_{m+1}\) respectively. Also, keep in mind the sequence of pairs of regions and edges traversed by the sequence \(((R_1, e_1), \dots, (R_l, e_l))\) which will be useful in the rest of the proof.

    Observe that each unique edge in \(\{e_1, \dots, e_l\}\) is associated with a \texttt{From} instruction within the \texttt{Do Until} block. As a first step, let us treat the indices \(i \in [1, l-1]\) such that \(e_{i+1} = e_i\). By Lemma \ref{lemma:three-time-visit}, this happens at most twice with \(e_i\) so this contributes at most two targets, and thereby two lines to the \texttt{From} instruction of \(e_i\). On the other hand, suppose \(e_{i+1} \neq e_i\). As a first subcase, suppose \(\actions(R_i) \subseteq \{\leftact, \rightact\}\) or \(\actions(R_i) \subseteq \{\upact, \downact\}\). From Algorithm \ref{alg:synth-policy} it is clear that such regions would add at most two targets (directions) to the \texttt{From \(e_i\)} instruction. Next assume at least two orthogonal directions are allowed in \(R_i\). If \(e_{i+1} \neq e_i\) at most once with \(e_i\), then it contributes only one target to the \texttt{From \(e_i\)} instruction. However, if there is another index \(j \in [1, l-1]\) such that \(e_i = e_j\) and \(e_{j+1} \neq e_j\), by Lemma \ref{lemma:edges-two-ortho-dirs}, it means that we have an inner or an outer loop. From this case, we would again have at most two targets: either \(e_{j+1} = e_{i+1}\) and so the target segments would be merged (and the loop continues) or \(e_{j+1}\) is a new edge never visited before (and the loop is exited). In other words, a loop contributes at most two target segments to each edge and there is at most one loop in a \texttt{Do Until} block.

    In total, we have at most six targets associated with each edge-region pair in the \texttt{Do Until} block. As each edge can be shared between two regions, at most twelve targets are associated with each edge. Further, since only \(m\) edges are explored by the \(m\)-th block, each block has at most \(12m\) instructions. Thus, we obtain the following bound on the total length of the policy
    \begin{equation}
        \sum_{m=1}^{\abs{\texttt{Regions}}^2} 12m = O(\abs{\texttt{Regions}}^4).
    \end{equation}
\end{proof}

\section{Analysis of the bitsize of programs: Proof of Lemma~\ref{lemma:size-segment}}
\begin{proof}
    Let \[\left[\left(\frac{x_1}{D}l, \frac{y_1}{D}l\right), \left(\frac{x_2}{D}l, \frac{y_2}{D}l\right)\right]\] be an edge of a region in \(\regions\) for some \(x_1, y_1, x_2, y_2 \in \llbracket 0, D \rrbracket\).

    Associated to this edge, we can write the two following equations for the line on which it lies on:
    \begin{align}
        y &= \frac{y_2 - y_1}{x_2 - x_1}x + \frac{(x_2 - x_1)y_1 - (y_2 - y_1)x_1}{(x_2 - x_1)D}l \\
        x &= \frac{x_2 - x_1}{y_2 - y_1}y + \frac{(y_2 - y_1)x_1 - (x_2 - x_1)y_1}{(y_2 - y_1)D}l
    \end{align}

    Note that when \(y_2 - y_1 = 0\) or \(x_2 - x_1 = 0\), one of the two equations does not exist. Observing that \((x_2 - x_1), (y_2 - y_1) \in \llbracket-D, D\rrbracket\), we have that \(H \coloneqq D(D!)\) is divisible by \((x_2 - x_1)D\) and \((y_2 - y_1)D\). So we can write
    \begin{align}\label{eq:y-x-equation-forms}
        y &= \frac{s}{H}x + \frac{d}{H}l \\
        x &= \frac{s'}{H}y + \frac{d'}{H}l
    \end{align}

    where
    \begin{equation}
        s \coloneqq H\frac{y_2 - y_1}{x_2 - x_1}
    \end{equation}

    \begin{equation}
        d \coloneqq H\frac{(x_2 - x_1)y_1 - (y_2 - y_1)x_1}{(x_2 - x_1)D}
    \end{equation}

    and similarly for \(s'\) and \(d'\). In particular, these numerators satisfy the following bounds
    \begin{equation}
        \abs{s}{} = \abs{H\frac{y_2 - y_1}{x_2 - x_1}} \leq H \abs{y_2 - y_1} \leq HD
    \end{equation}
    \begin{equation}
        \abs{d}{} = \abs{H\frac{(x_2 - x_1)y_1 - (y_2 - y_1)x_1}{(x_2 - x_1)D}} \leq \frac{H}{D}(\abs{(x_2 - x_1)y_1 + (y_2 - y_1)x_1}) \leq \frac{H}{D}2D^2 = HD
    \end{equation}

    With this, we are now able to write the equation for the line containing each of the edges as shown in \ref{eq:y-x-equation-forms}. These two forms of the equation are relevant to us because each time we are extending a node of a segment \([a_i, b_i]\) in a tree of the winning region, we are computing \([a_{i-1}, b_{i-1}]\) by intersecting an edge with the half-planes reachable by the allowed actions in \(R_{i-1}\) which amounts to finding the intersection points of the line containing \(e_{i-1}\) with a certain horizontal or vertical line. So we can substitute the value of the \(x\) or \(y\) coordinate in \ref{eq:y-x-equation-forms} to obtain the endpoints of \([a_{i-1}, b_{i-1}]\). Notice that filtering out explored parts of edges only uses precomputed intersection points.

    For example, if while extending the segment \([a_p, b_p]\), to compute \([a_{p-1}, b_{p-1}]\), we have to intersect with the line \(x = (a/H)l\) where \(a \in \llbracket 0, H \rrbracket\). Substituting this into \ref{eq:y-x-equation-forms}, gives us
    \begin{equation}
        y_{p-1} = \frac{s}{H} \frac{a}{H} l + \frac{d}{H} l = \frac{sa + dH}{H^2}l
    \end{equation}

    with \(\abs{sa + dH}{} \leq H^2D\). Now suppose while extending the segment \([a_{p-1}, b_{p-1}]\) to \([a_{p-2}, b_{p-2}]\), we have to intersect with the line \(y = \frac{u}{H^2}l\) where \(u \in \llbracket 0, H^2D \rrbracket\). Then, in the same way,
    \begin{equation}
        x_{p-2} = \frac{v}{H^3}
    \end{equation}
    for some \(v \in \llbracket 0, H^3D^2 \rrbracket\).

    Continuing this argument inductively, we get that the endpoints of the first segment \([a_1, b_1]\) can be written with coordinates of the form \[\frac{t}{H^p}\] where \(t \in \llbracket 0, H^{p}D^{p-1} \rrbracket\). Furthermore, since
    \begin{equation}
        H^pD^{p-1} = D^{2p-1} (D!)^p = O(D^{p(D+2) - 1})
    \end{equation}

    So, in order to store \([a_1, b_1]\) which potentially requires more space to store than any of the other segments, we need to store a few integers in \(\llbracket 0, H^{p}D^{p-1} \rrbracket\) which requires
    \begin{equation}
        \log(H^pD^{p-1}) = O((p(D+2) - 1)\log(D)) = O(pD\log(D))
    \end{equation}
    bytes.
\end{proof}

\section{Implementation and evaluation}
\label{sec:implementation}
Together with this article we release a small Python package\footnote{\url{https://github.com/guruprerana/smol-strats}} including modules for generating gridworld instances as well as implementation of the algorithms for the construction of the tree of the winning region and synthesis of policies. 

Gridworlds are continuous environments. In practice, we use a discretised version: the state space is a \(n \times n\). The regions are defined by linear predicates.

The \texttt{linpreds} module contains classes to generate random gridworlds. This is done by choosing at random linear predicates on a \(n \times n\) grid where the endpoints of the linear predicates are in \([0, n-1]\). We then assign random actions to each of the regions that are created by the intersections of these linear predicates. It also includes functions to generate a PRISM program from the gridworld.

The \texttt{polygons} module contains the infrastructure to translate the linear predicates generated into a data structure which makes the backward winning region construction efficient. We use the half-edge data structure (popular in computational geometry) by looking at the gridworld as a planar tiling of the grid with polygons. The \(\texttt{backward\_reachability}\) module constructs the tree of the winning region. The \texttt{game.continuous} module implements a reinforcement learning-like game environment which can simulate a policy for gridworld instances. Lastly, \texttt{policy.subgoals} implements Algorithm~\ref{alg:synth-policy} and can synthesize programmatic policies from a path of segments.

The \texttt{benchmarks} folder contains a set of 17 benchmarks including the spiral and double-pass triangle examples. The others were generated by our code and go up to instances with 50 linear predicates and 600+ regions. The synthesized policies can be seen and the policy path visualized in images in the respective folders of the benchmarks. The benchmark data can be found in Table~\ref{table:benchmarks}. 
We measure size in bytes to take into account the size of numerical coefficients involved. 
We observe that the size of the policy is polynomial (almost linear) in the size of the gridworld. Note that the size of the gridworld is the space required to store all the edges of all the regions of the gridworld.

\begin{table}[]
\centering
\begin{tabular}{|llll|}
\hline
\textbf{Benchmark}               & \textbf{Gridworld size} & \textbf{Policy size} & \textbf{Regions} \\
\hline
spiral                  & 10833        & 20847       & 14      \\
size3preds5loopy        & 8786         & 15744       & 11      \\
size50preds10-1         & 30669        & 57926       & 32      \\
size50preds20-1         & 105252       & 126263      & 113     \\
size100preds20-1        & 119253       & 136257      & 126     \\
size100preds20-2        & 108256       & 130591      & 115     \\
size100preds30-1        & 228676       & 256031      & 233     \\
size100preds30-2        & 220557       & 248063      & 230     \\
size100preds30-3        & 221308       & 244846      & 227     \\
size100preds30-4        & 266882       & 303592      & 271     \\
size100preds50-1        & 612940       & 655357      & 616     \\
size100preds50-2        & 668670       & 706836      & 668     \\
size100preds50-3        & 635978       & 663439      & 628     \\
size100preds50-4        & 538503       & 576528      & 542     \\
size100preds50-5        & 603314       & 641681      & 616  \\
\hline
\end{tabular}
\caption{Size of synthesized policies (in bytes) for a set of generated benchmarks.}
\label{table:benchmarks}
\end{table}

\begin{figure*}
    \centering
    \begin{subfigure}{0.6\linewidth}
        \includegraphics[width=\linewidth]{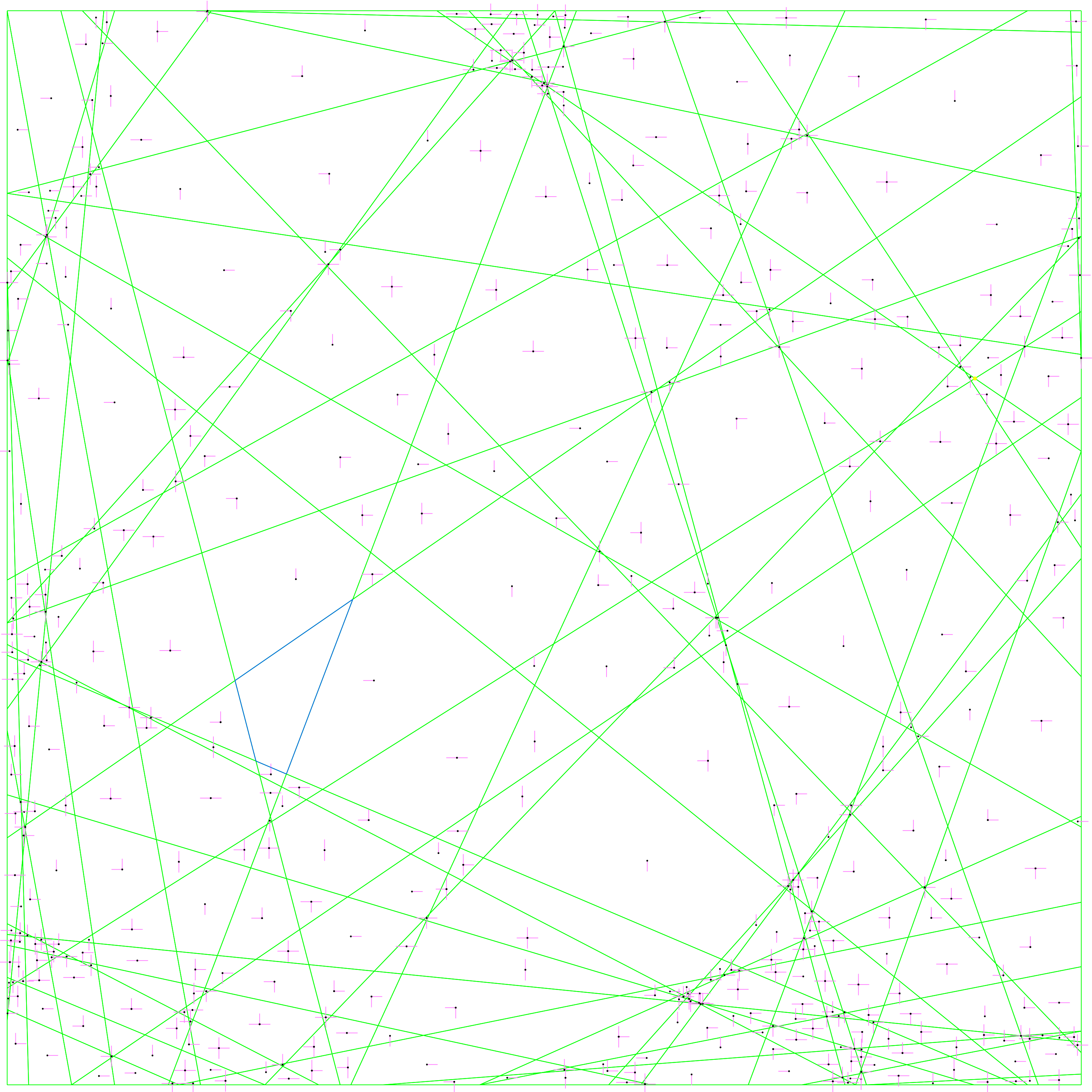}
        \caption{A gridworld}
    \end{subfigure}
    \begin{subfigure}{0.6\linewidth}
        \includegraphics[width=\linewidth]{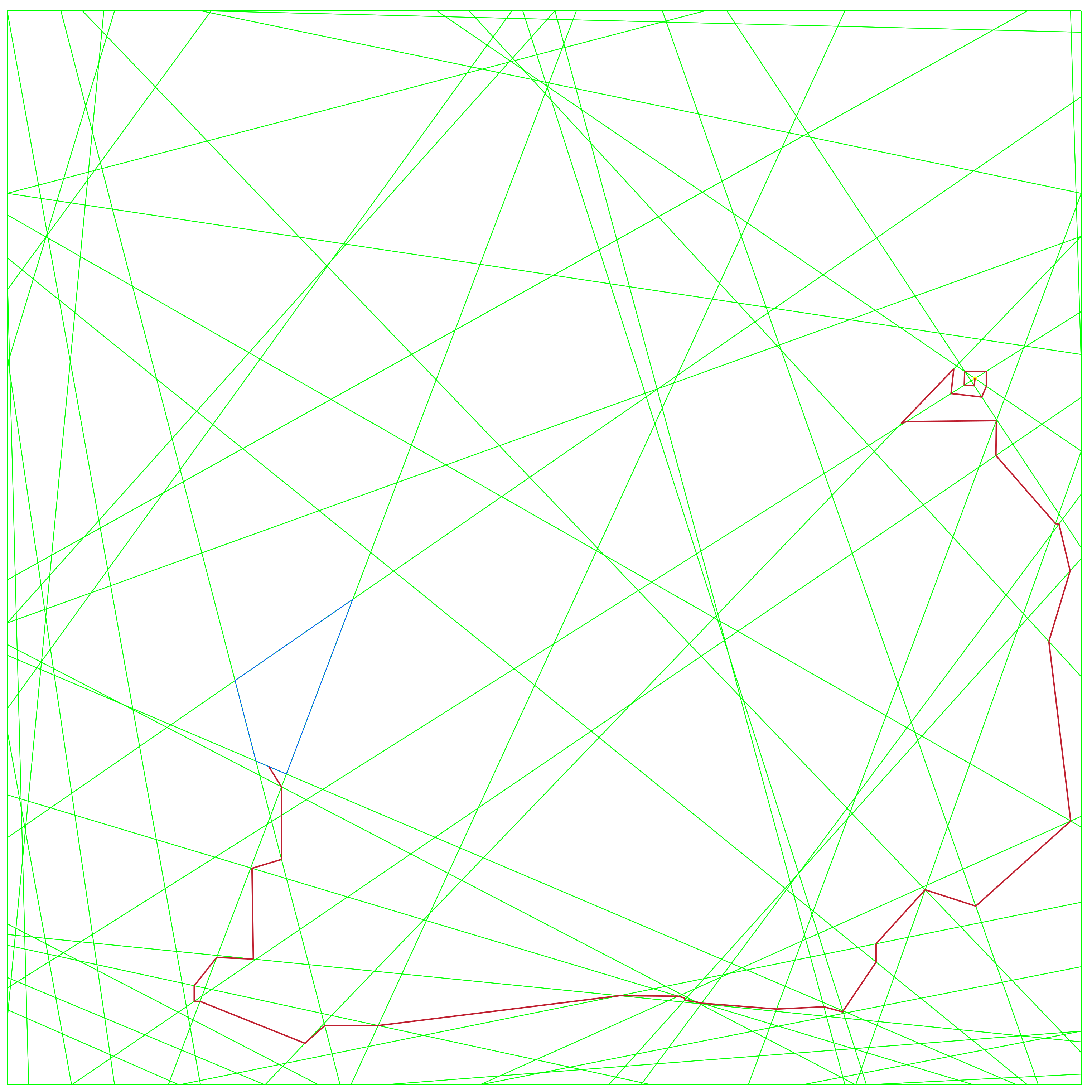}
        \caption{Synthesized policy path}
    \end{subfigure}
    \caption{The \texttt{size100preds50-4} benchmark is a gridworld with 542 regions.}
    \label{fig:size100preds50-4}
\end{figure*}

\end{document}